\documentclass[10pt]{article}

\PassOptionsToPackage{dvipsnames}{xcolor}

\usepackage[margin=1in]{geometry}
\usepackage[numbers,sort&compress]{natbib}

\usepackage{iftex}
\ifPDFTeX
  \usepackage[utf8]{inputenc} %
  \usepackage[T1]{fontenc}    %
  \usepackage{CJKutf8}
\fi
\ifLuaTeX
  \usepackage{luatexja}
  \usepackage{luatexja-fontspec}
  \setmainjfont{Noto Sans Mono CJK JP}
\fi
\ifXeTeX
  \usepackage{fontspec}
  \usepackage{xeCJK}
\fi
\usepackage{hyperref}       %
\usepackage{url}            %
\usepackage{booktabs}       %
\usepackage{amsfonts}       %
\usepackage{nicefrac}       %
\usepackage{microtype}      %
\usepackage[dvipsnames]{xcolor}
\usepackage[normalem]{ulem}

\makeatletter
\@tfor\@L:=ABCDEFGHIJKLMNOPQRSTUVWXYZ\do{%
  \edef\@tmp{\noexpand\newcommand\expandafter\noexpand\csname b\@L\endcsname{\noexpand\mathbb{\@L}}}%
  \@tmp
}
\@tfor\@L:=ABCDEFGHIJKLMNOPQRSTUVWXYZ\do{%
  \edef\@tmp{\noexpand\newcommand\expandafter\noexpand\csname c\@L\endcsname{\noexpand\mathcal{\@L}}}%
  \@tmp
}
\makeatother

\newcommand{\bx}{\mathbf{x}}

\usepackage{amsmath}
\usepackage{physics}
\usepackage{amsthm}
\newtheorem{definition}{Definition}
\newtheorem{proposition}{Proposition}

\usepackage{thmtools}
\usepackage{amssymb}
\usepackage{graphicx}
\makeatletter
\newcommand{\HeatmapInclude}[2][]{%
  \IfFileExists{#2}{%
    \includegraphics[#1]{#2}%
  }{%
    \begingroup
    \setlength{\fboxsep}{6pt}%
    \fcolorbox{gray!40}{gray!15}{%
      \parbox[c][2.6cm][c]{0.96\linewidth}{%
        \centering\footnotesize\sffamily
        \textcolor{gray!60}{\textbf{[missing figure]}}\\[0.35em]
        \texttt{\detokenize{#2}}%
      }%
    }%
    \endgroup
  }%
}
\makeatother
\usepackage{wrapfig}
\usepackage{subcaption}
\usepackage{bbm}
\usepackage{multirow}
\usepackage{tabularx}
\usepackage[nameinlink,capitalize]{cleveref}
\crefname{section}{Section}{Sections}
\Crefname{section}{Section}{Sections}
\crefname{subsection}{Section}{Sections}
\Crefname{subsection}{Section}{Sections}
\crefname{appendix}{Section}{Sections}
\Crefname{appendix}{Section}{Sections}
\crefname{equation}{Equation}{Equations}
\Crefname{equation}{Equation}{Equations}
\crefname{figure}{Figure}{Figures}
\Crefname{figure}{Figure}{Figures}
\crefname{table}{Table}{Tables}
\Crefname{table}{Table}{Tables}
\crefname{subfigure}{Figure}{Figures}
\Crefname{subfigure}{Figure}{Figures}
\crefname{subtable}{Table}{Tables}
\Crefname{subtable}{Table}{Tables}
\crefname{proposition}{Proposition}{Propositions}
\Crefname{proposition}{Proposition}{Propositions}
\crefname{corollary}{Corollary}{Corollaries}
\Crefname{corollary}{Corollary}{Corollaries}
\crefname{lemma}{Lemma}{Lemmas}
\Crefname{lemma}{Lemma}{Lemmas}
\crefname{definition}{Definition}{Definitions}
\Crefname{definition}{Definition}{Definitions}
\title{Learning Large-Scale Modular Addition with an Auxiliary Modulus}
\date{}

\author{%
  Hanato Kikuchi$^{1}$, Ryosuke Masuya$^{1}$, Kazuhiko Kawamoto$^{1}$,\\[0.5em]
  Hiroshi Kera$^{1,2}$\thanks{Corresponding author: Hiroshi Kera (\texttt{kawa@chiba-u.jp}).}\\[0.55em]
  $^{1}$Chiba University \qquad $^{2}$National Institute of Informatics
}

\begin{document}
\ifPDFTeX
\begin{CJK*}{UTF8}{min}
\fi

\maketitle

\begin{abstract}
  Learning parity functions, more general modular addition, is a challenging machine learning task due to its input sensitivity.
  A recent study substantially scaled modular addition learning in both the number of summands and the modulus. Its key idea is to increase zeros in training sequences, reducing the effective number of summands and thus controlling training difficulty; however, this induces covariate shift between training and test input distributions.
  This study theoretically and empirically analyzes this side effect and proposes a covariate-shift-free method for modular addition.
  Specifically, we introduce an auxiliary modulus $Kq$ during training, which reduces wrap-around frequency and problem difficulty while preserving the same input distribution across training and testing.
  Experiments show strong scalability and sample efficiency: even for large input length $N$, large modulus $q$, and small datasets---where the sparse method fails to learn---our method achieves equal or better match accuracy and relaxed $\tau$-accuracy. For example, at $N=64$ and $q=974269$, our method trained on 100K samples achieves $97.0\%$ $\tau$-accuracy at $\tau=0.05$, while the sparse method achieves only $9.5\%$ with the same data size and $93.9\%$ even when extended to 1M samples.
\end{abstract}

\section{Introduction} \label{ch:introduction}

Despite recent success in learning complex symbolic computational tasks, such as symbolic integration~\citep{charton2021learning} and Gr\"obner basis computation~\citep{DBLP:conf/nips/KeraIKVY24}, and so on~\citep{kera2026computational,alfarano2024global}, learning modular addition is still a challenging task for deep learning models.
\begin{definition}[Modular Addition]\label{def:modular-addition}
    Let $N, q \in \bN$ and $q \geq 2$. Given a sequence $\mathbf{x}=[x_1,\ldots,x_N]$ with $x_i\in\{0,\ldots,q-1\}$, compute the sum $\sum_{i=1}^N x_i$ modulo $q$, i.e., $f_{q}(\mathbf{x}) := \sum_{i=1}^N x_i \mod q$.
\end{definition}
The difficulty of learning $f_q$ arises from the input sensitivity: even a single change in the input can significantly affect the output. Even a single change in the input can significantly affect the output. Theoretical analysis by \citet{pmlr-v70-shalev-shwartz17a} and \citet{hahn2024sensitivee} has proved this for the parity problem (i.e., modular addition with $q=2$). Larger $N, q$ increases the difficulty further as it incurs more wraps around the modulus. The number of required training samples grows with $N$ and $q$ \citep{conf/icml/MohamadiLWS24}; besides, large $N$ necessitates an increased network width \citep{DBLP:journals/corr/grokkingmojularpoly}.

A recent study by \citet{conf/icml/SaxenaAWL25} identified learning modular addition as a key component of attacking the Learning With Errors (LWE) problem, a basis for post-quantum cryptography \citep{NIST2024}, and proposed a learning method that substantially scaled modular addition learning in both the number of summands and the modulus. Particularly, addressing large modulus $q$ is crucial for their application. A naive training even leads to stagnated training loss because of the small variance of the gradient of the loss function~\cite{pmlr-v70-shalev-shwartz17a}. Thus, the key is to limit wrap-around frequency at training time and ease the problem difficulty to make training flow. \citet{conf/icml/SaxenaAWL25} implemented this idea by increasing zeros in training sequences, thereby lowering the effective number of summands; however, this induces covariate shift between training and test input distributions, which is unfavorable from a standard machine learning perspective.

In this study, we propose a new learning method for modular addition that involves no covariate shift and attains scalable and sample-efficient learning.
The proposed method introduces an auxiliary loss of modular addition with enlarged modulus $Kq$ for $K>1$. The enlarged modulus reduces wrap-around frequency and problem difficulty at training time, while preserving the input distribution for both training and testing.
Our experiments demonstrate that elimination of covariate shift substantially improves the sample efficiency and scalability of modular addition learning: the proposed method achieves competitive or superior accuracy compared to the sparse method by \citet{conf/icml/SaxenaAWL25} under various conditions of $N, q$, and even with a smaller training sample size.

We summarize the contributions as follows.
\begin{itemize}
    \item  We examine the sparse method by \citet{conf/icml/SaxenaAWL25} and theoretically and empirically demonstrate the side effect of covariate shift. Specifically, training on sparse inputs with more zeros than dense inputs in the test distribution induces non-trivial generalization error. Besides, the success of the sparse method is sensitive to the model configuration, such as Dropout, PreNorm, Bias, and weight initialization.
    \item  We propose a covariate-shift-free method for modular addition learning. An auxiliary loss of modular addition with enlarged modulus $Kq$ for $K>1$ is introduced. This loss replaces the main loss with probability of $r$. The enlarged modulus reduces wrap-around frequency and problem difficulty at training time, while preserving the input distribution for both training and testing.
    \item Our method scales to larger problem sizes beyond the reach of the prior method. Concretely, at $N=64$ and $q=974269$, training on only 100K samples yields $97.0\%$ $\tau$-accuracy (with relative error tolerance $\tau=0.05$), whereas the sparse method attains $9.5\%$ at the same scale and only $93.9\%$ even with 1M samples.
\end{itemize}

\section{Related Work} \label{ch:relatedwork}
We summarize related work on learning modular arithmetic using deep learning models, methods for utilizing easier examples, and auxiliary tasks.

\paragraph{Learning modular addition with deep learning models.}
Recent studies have investigated how deep learning models learn modular addition.
These works show that models learn this task by acquiring internal periodic structures.
For instance, this periodicity is confirmed in algorithm models learned for modular addition. Specifically, it has been shown that models treat numbers as points in Fourier space and learn modular addition through algorithms such as addition of angles \citep{DBLP:conf/iclr/NandaCLSS23}, the "Pizza" algorithm \citep{DBLP:conf/nips/ZhongLTA23}, and the approximate Chinese Remainder Theorem \citep{DBLP:journals/corr/abs-2505-18266}.
Additionally, the weights and embeddings of models trained on modular addition also exhibit periodic structures \citep{Power2022grokking,Liu2022groking}, and periodic analytical solutions for these weights have been discovered \citep{DBLP:journals/corr/grokkingmodulerarithmetic}. Furthermore, these periodic structures are observed in other modular arithmetic tasks, such as modular multiplication \citep{DBLP:journals/corr/grokkingmojularpoly,furuta2024towards}.
Therefore, helping models learn periodic structures is important. Previous methods, such as angle embeddings \citep{stevens2025salsa}, directly add periodicity by treating numbers as points on a circle. Instead of changing the inputs, our method uses an auxiliary task with a different period (modulus $Kq$) to help the model learn periodicity.

\paragraph{Learning from simpler tasks.}
Utilizing simpler tasks is effective for learning modular addition and parity problems. For instance, one approach uses Chain of Thought \citep{DBLP:conf/nips/Wei0SBIXCLZ22} to decompose problems into steps, treating them as simpler sub-problems. This enables models to solve parity problems that are otherwise difficult to learn \citep{DBLP:conf/iclr/WiesLS23,DBLP:conf/iclr/KimS25}. Another effective approach is curriculum learning \citep{DBLP:conf/icml/BengioLCW09}, which gradually increases task difficulty. For parity problems, sparse parity problems \citep{DBLP:conf/nips/DanielyM20} are utilized as simpler tasks, and applying curriculum learning with them improves both learning accuracy and efficiency on the target problem \citep{DBLP:conf/nips/AbbeCL23}. Similarly, for modular addition, recent work mixes sparse and standard tasks during training, successfully learning more complex modular addition (e.g., up to a terms of $N = 128$ modulus of $p=974269$) than conventional methods \citep{conf/icml/SaxenaAWL25}.
However, the method in \citet{conf/icml/SaxenaAWL25} causes covariate shift because mixing easier problems only during training alters the input distribution (as detailed in \cref{ch:background}). In contrast, our method avoids this by using an auxiliary task to learn both problems simultaneously, keeping the input distribution consistent.

\paragraph{Learning with auxiliary tasks.}Auxiliary task learning improves generalization \citep{caruana1997multitask, ruder2017overview}. In mathematical tasks, learning different operations simultaneously improves accuracy. For example, joint training on multiple arithmetic operations enhances performance \citep{lee2024teaching}. Furthermore, combining simple and difficult tasks reduces model size if they share core building blocks \citep{both2025small}. While \citet{both2025small} simultaneously train on different operations (e.g., modular addition and modular multiplication with a non-prime modulus), we use a highly related task: the same modular addition with a different modulus.

\section{Problem Setting} \label{ch:background}

This study tackles modular addition learning (cf.~\cref{def:modular-addition}). Concretely, we adopt the setup below.

\paragraph{Training on sparse inputs.} A very recent state-of-the-art method for learning modular addition---which we refer to as the sparse method— is to train on inputs that are sparser than samples from the uniform distribution on \(\{0,\dots,q-1\}^N\)\citep{conf/icml/SaxenaAWL25}. Concretely, the procedure first samples the number of non-zero entries \(z\in\{1,\dots,N\}\) from \(g(z)\propto 1/\sqrt{N-z+1}\), then chooses \(z\) positions uniformly at random and fills them independently with values from \(\{0,\dots,q-1\}\), leaving the remaining \(N-z\) positions as zero. While \(g\) itself peaks at \(z=N\) (a fully populated input), it has a much heavier tail toward small \(z\) than the uniform marginal \(\mathrm{Bin}(N,(q-1)/q)\); hence the resulting samples are typically sparser than uniform ones. Involving more zeros implies fewer wraps around the modulus, thereby reducing the difficulty of the problem. 

While the sparse method is empirically effective, it seems unnatural because it introduces a covariate shift between the training distribution (many zeros in $\bx$) and the test distribution (uniformly sampled from $\{0,\dots,q-1\}^N$). To examine this, we quantify the resulting generalization gap.

\begin{proposition}[Generalization gap under sparse training, informal]\label{prop:gap-sparse}
    Let \(P_{\mathrm{te}}\) be the uniform distribution on \(\mathcal{X}=\{0,\dots,q-1\}^N\) and 
    let \(P_{\mathrm{tr}}\) be the sparse construction of \mbox{\citep{conf/icml/SaxenaAWL25}}. 
    Write \(n_0(x):=\#\{i:x_i=0\}\) for the number of zero entries, and call an input 
    \emph{zero-free} if \(n_0(x)=0\) (all entries are non-zero). 
    Suppose the model attains training risk \(R_{\mathrm{tr}}(f)\le\delta\) but its expected loss 
    on a uniformly random zero-free input is at least \(\varepsilon\). Then
    \begin{align}
        R_{\mathrm{te}}(f)-R_{\mathrm{tr}}(f)
        \;\ge\;
        \varepsilon\!\left(1-\tfrac{1}{q}\right)^{\!N}-\delta.
        \label{eq:gap-tv}
    \end{align}
\end{proposition}
\textit{Proof sketch.} Conditioning on \(n_0(x)\) gives 
\(R_{\mathrm{te}}(f)\ge P_{\mathrm{te}}(n_0(x)=0)\,\varepsilon=(1-1/q)^N\varepsilon\); 
subtracting \(R_{\mathrm{tr}}(f)\le\delta\) yields \cref{eq:gap-tv}. See 
\cref{sec:proof-gap-sparse} for the formal statement and full proof.

\begin{figure}[htbp]
    \centering
    \begin{subfigure}[b]{0.48\textwidth}
        \centering
        \includegraphics[width=\textwidth]{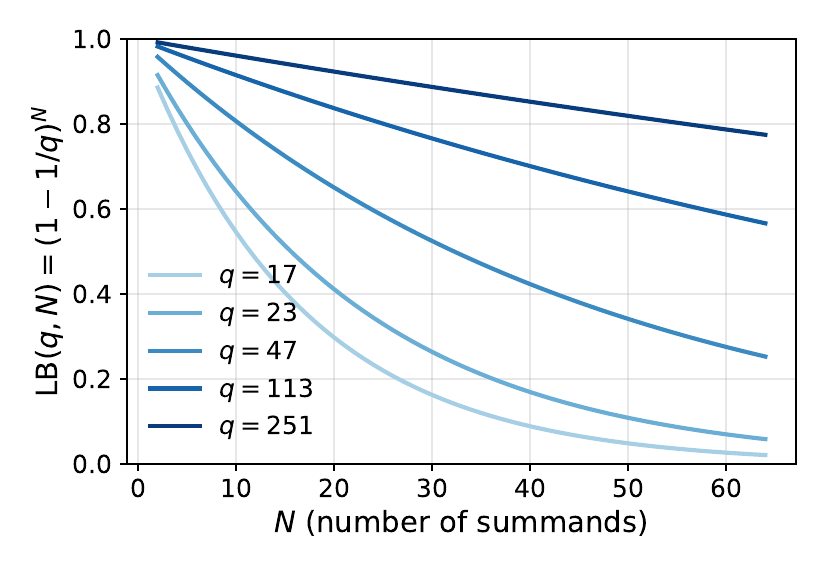}
        \caption{Generalization gap lower bound vs.\ $N$.}
        \label{fig:gap-lb-vs-N}
    \end{subfigure}
    \hfill
    \begin{subfigure}[b]{0.48\textwidth}
        \centering
        \includegraphics[width=\textwidth]{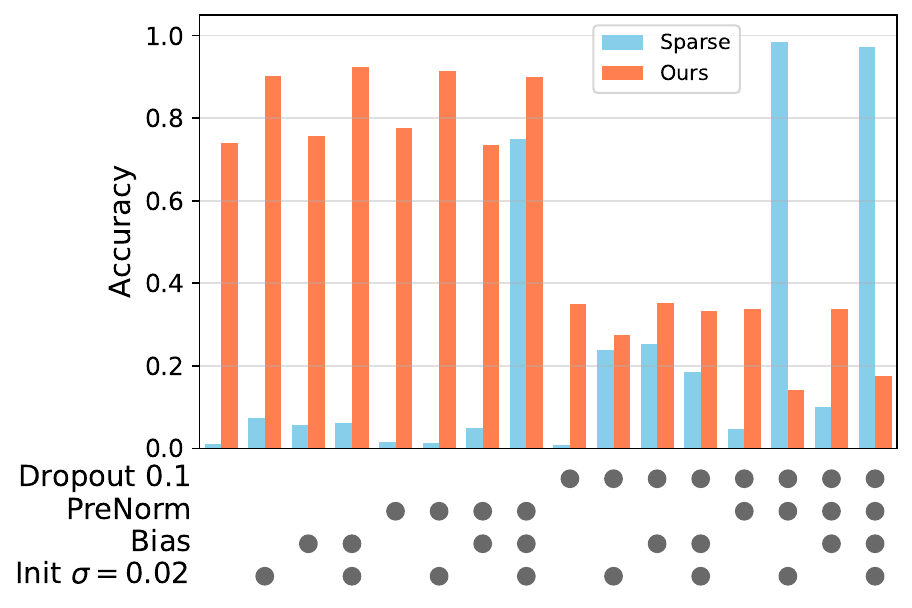}
        \caption{Match accuracy vs.\ model configuration.}
        \label{fig:param_sweep_comparison}
    \end{subfigure}
    \caption{%
      \textbf{(a)} Lower bound \(\mathrm{LB}(q,N)=(1-1/q)^N\) of \cref{eq:gap-tv} as a function of \(N\) for several moduli \(q\). Larger \(q\) (darker blue) keeps the bound substantial across the whole range, while for small \(q\) the bound vanishes as \(N\) grows. See \cref{sec:proof-gap-sparse} for the underlying table.
      \textbf{(b)} Comparison of match accuracy under various model configurations ($N=8$, $q=97$). The lower panel indicates the active components for each configuration; the upper panel compares the sparse method (blue) and ours (orange).%
    }
    \label{fig:main_figure_label}
\end{figure}

The prefactor \((1-1/q)^N\) is large in typical modular-addition regimes 
(\(0.70\) for \((q,N)=(23,8)\), \(0.93\) for \((113,8)\), \(0.75\) for \((113,32)\); see 
\cref{fig:gap-lb-vs-N}) and tends to \(1\) as \(N/q\to 0\), in which case \cref{eq:gap-tv} 
reduces to \(R_{\mathrm{te}}(f)-R_{\mathrm{tr}}(f)\gtrsim\varepsilon-\delta\): any non-trivial 
error on the zero-free inputs that dominate the test distribution translates into an irreducible 
test-time gap.

\paragraph{Model configuration sensitivity.}
Beyond this theoretical bound, we observe empirically that the sparse method is 
sensitive to the model configuration (see \cref{fig:param_sweep_comparison}). 
Specifically, when varying the active components—such as Dropout, PreNorm, Bias, and weight initialization—the performance of the sparse method frequently collapses, often failing to learn entirely. In contrast, the proposed method is highly robust to such structural changes. Further details and comprehensive evaluations are provided in \cref{app:model_stability}. Although we do not have direct evidence that the covariate shift causes this, it is reasonable to consider that learning under covariate shift requires a careful training setup.

Note that the discussion above should not be read as a criticism of the sparse-input work of 
\citep{conf/icml/SaxenaAWL25}. 
Modular addition is hard enough that, without biasing training towards easier instances, 
even the training loss fails to decrease within practical compute 
budgets. Controlling problem difficulty is, therefore, necessary; they already 
demonstrated one concrete way to do this. 
What we pinpoint is only a side 
effect of \emph{how} this control is realized---inserting many zeros creates a covariate shift 
that bottlenecks the test risk.

Therefore, the desirable direction is to control the difficulty of training problems 
without introducing a covariate shift; we next present a method that achieves this.

\section{Method} \label{ch:method}

We now present our learning pipeline for the modular addition task. An important, well-known fact is that learning addition over $\mathbb{Z}$ is easy while learning addition over $\mathbb{Z}/q\mathbb{Z}$ is hard. That is to say, the difficulty of the problem derives from wrapping around the modulus.

\subsection{Auxiliary modulus $Kq$}
The central idea is to reduce the number of wraps around the modulus during training to facilitate learning. \citet{conf/icml/SaxenaAWL25} implemented this idea by sparsifying the input $\mathbf{x}$, but this leads to a covariate shift between the training and testing. Instead, we introduce an auxiliary modulus $Kq,\ (K > 1)$ to reduce the number of wraps around the modulus during training. Formally, with $r \in [0, 1]$, we define the target label $y$ for input $\mathbf{x}$ as follows:
\begin{align}\label{eq:target_label_aux}
    y = \begin{cases}
f_q(\mathbf{x})  & \text{with probability } 1-r \\
f_{Kq}(\mathbf{x}) & \text{with probability } r.
\end{cases}
\end{align}
Namely, we train the model to predict the solution to modular addition with modulus $q$ with probability $1-r$ and the solution to modular addition with modulus $Kq$ with probability $r$.

Importantly, the input remains restricted to $\{0,\dots,q-1\}$, i.e., $\mathbf{x} \in \{0,\dots,q-1\}^N$, for the auxiliary modulus $Kq$. The number $D_{Kq}(\bx)$ of wraps around the (extended) modulus for $\bx$ is consequently reduced to 
\begin{align}\label{eq:wrap_function}
    D_{Kq}(\bx) = \left\lfloor\frac{\sum_{i=1}^N x_i}{Kq}\right\rfloor,
\end{align}
where $\lfloor\, \cdot\, \rfloor$ denotes the floor function. 
The expected number of wraps around the modulus is then $\bE[(1-r)D_q + r D_{Kq}]$.

\subsection{Analysis of problem difficulty reduction}\label{subsec:analysis}
We examine the reduction of the problem difficulty by the proposed method and the sparse method. Here, the problem difficulty is measured by the expected number of wraps around the modulus. For simplicity of analysis, we ignore the floor function. 

We keep the notation as in \cref{ch:method}. The (continuous) number of wraps around the modulus is the expectations of the following random variables:
   \begin{align*}
        X_0 &= \frac{\sum_{i=1}^N x_i}{q}, &
        X_1 &= (1-r)\frac{\sum_{i=1}^N x_i}{q} + r\frac{\sum_{i=1}^N x_i}{Kq}, &
        X_2 &= \frac{\sum_{i=1}^z x_i}{q}, &
   \end{align*}  
   where $z$ is a random variable that represents the number of non-zero entries of an input data, and has a distribution $f_{\mathrm{inv\_sqrt}}(z) \propto 1/\sqrt{N - z + 1}$ introduced in \citet{conf/icml/SaxenaAWL25}, and $x_i$'s are i.i.d. random variables with a discrete distribution $U_q(x)=1/q$ if $x\in \{0, \ldots, q-1\}$ otherwise $0$. The random variable $X_0$ represents the number of wraps around the modulus, corresponding to the naive method without any training difficulty reduction. The random variables $X_1$ and $X_2$ correspond to the proposed method and the sparse method, respectively. 
   \begin{wrapfigure}{r}{0.4\textwidth}
      \centering
      \includegraphics[width=0.95\linewidth]{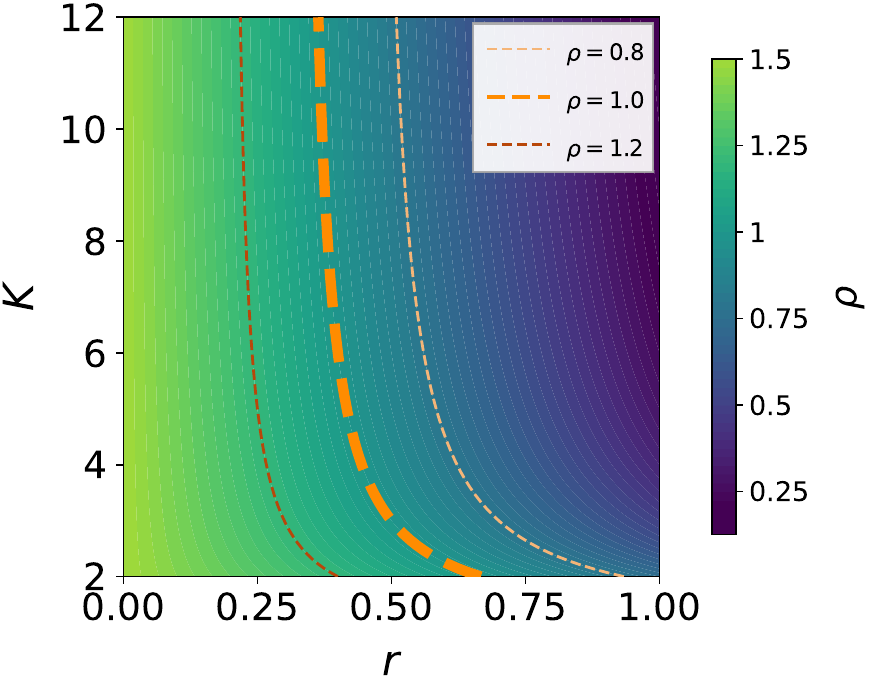}
      \captionsetup{skip=2pt}
      \caption{Heatmap of $\rho=\bigl((1-r)+r/K\bigr)/(2/3)$ in~\cref{eq:wrap_scale_factor_ratio}. The orange contour is defined by $\rho=1$. The range of $\rho$ is from $0.0$ to $1.5$.}
      \label{fig:wrap_scale_factor_heatmap}
      \vspace{-0.4\baselineskip}
   \end{wrapfigure}
   Then, we have the following proposition:
   \begin{proposition}\label{prop:expect_num_wraps}
    Keep the above notation. For $X_0$ and $X_1$, we have
    \begin{align*}
        \bE[X_0] &= \frac{N(q-1)}{2q}, \\
        \bE[X_1] &= \left((1-r) + r\frac{1}{K}\right)\frac{N(q-1)}{2q}
        =\left((1-r) + r\frac{1}{K}\right)\bE[X_0].
    \end{align*}
    For $X_2$, when $N$ is large enough, we have
    \begin{align*}
        \bE[X_2] &\approx \frac{2}{3}\frac{N(q-1)}{2q}=\frac{2}{3}\bE[X_0].
    \end{align*}
    \vspace{-0.3\baselineskip}
   \end{proposition}
   \Cref{prop:expect_num_wraps} shows that all expectations have the same factor $N(q-1)/2q\ (= \bE[X_0])$, so the dependence on $N$ and $q$ is identical at this coarse level up to constant multipliers. This indicates that the proposed method realizes a similar reduction of the problem difficulty as the sparse method, without introducing a covariate shift. \Cref{fig:wrap_scale_factor_heatmap} shows the ratio 
   \begin{align}\label{eq:wrap_scale_factor_ratio}
    \rho = \qty((1-r) + r\frac{1}{K})\ \bigr/ \bigl.\  \frac{2}{3}
   \end{align}
   of the constant factor between the proposed method and the sparse method over $K, r$.
   We observe $\rho \approx 1$ for the combination of $K$ and $r$ that perform well in the experiments. 
   
   A full proof of \cref{prop:expect_num_wraps} and additional analysis of the wrap-count variable are given in \cref{sec:detail_analysis}. While this section focuses on the random variable $\sum_{i=1}^Nx_i/(Kq)$ without the floor function, \cref{sec:detail_analysis} also examines properties of the original discrete random variable $D_{Kq}$.

\subsection{Computational pipeline}\label{sec:computational_pipeline}

We here assume the Transformer model for the explanation clarity, yet our method is not limited to this model. The input $\bx = [x_1, \ldots, x_N]$ is tokenized into a sequence of $N$ tokens $[\texttt{x$_{x_1}$}, \ldots, \texttt{x$_{x_N}$}]$, where \texttt{x$_{n}$} is the token of scholar $n$. The embedding layer maps each token to an embedding vector. Transformer encoder layers process them, and finally, the classification layer maps the final hidden state to logits. We refer to this standard formulation as token embedding.

\paragraph{Approach 1: extending vocabulary.} The standard token embedding is used. Introducing the auxiliary modulus $Kq$ needs two minor modifications to the standard training pipeline. First, the vocabulary set is extended from $\{0, \ldots, q-1\}$ to $\{0, \ldots, Kq-1\}$. Second, one of the two labels $f_q(\bx)$ and $f_{Kq}(\bx)$ with probability $r$ and $1-r$, respectively, is selected for the classification loss (i.e., the cross-entropy loss), as shown in \cref{eq:target_label_aux}. In inference, only the first $q$ logits are used to predict the solution to modular addition with modulus~$q$.

While this approach is simple, it can suffer from a large vocabulary size when $K$ is large. To address this, we extend the angular embedding approach \citep{stevens2025salsa} to handle the auxiliary modulus $Kq$.

\paragraph{Approach 2: dual angular embedding.} Angular embedding~\citep{stevens2025salsa} embeds a scalar $x \in \{0, \ldots, q-1\}$ to a two-dimensional unit circle $(\cos \phi, \sin \phi)$ with $\phi = 2\pi x_i / q$ and then lifts it to the embedding space by a linear layer. The linear read-out layer maps the final hidden state to a 2-dimensional vector. The MSE loss is used to measure the angular error, instead of the classification loss. In our case, we use two moduli $q$ and $Kq$, and thus scalars are embeded to four-dimensional unit circle $(\cos \phi, \sin \phi, \cos \phi', \sin \phi')$ with $\phi' = 2\pi x_i / Kq$. The read-out layer maps the final hidden state to a 4-dimensional vector. In conclusion, only the first two dimensions are used to predict the solution to modular addition with modulus $q$.

\section{Experiment} \label{ch:experiment}
In this section, we evaluate the proposed method on the modular addition task across various numbers of terms $N$ and moduli $q$. Particularly, we show that our method attains comparable or higher match accuracy than the sparse method \citep{conf/icml/SaxenaAWL25} using only one-tenth of the training samples (\cref{ex:ex1}) for most of the settings.

\subsection{Exact modular addition}\label{ex:ex1}

First, we evaluate the proposed method on the exact modular addition task, where models are trained to predict the exact solution to modular addition. Most experimental setups follow \citep{conf/icml/SaxenaAWL25} unless otherwise emphasized.

\paragraph{Dataset.} For the training dataset $\mathcal{D}_{\text{train}}$ parameterized by $N$, $q$, and $|\mathcal{D}_{\text{train}}|$, each entry of the input $\mathbf{x}$ was sampled from a uniform distribution for our method and from the sparse distribution for the sparse method~\citep{conf/icml/SaxenaAWL25}.
The test dataset $\mathcal{D}_{\text{test}}$ was fixed to 1M uniformly generated samples. For each input $\mathbf{x}$, the dataset provided the primary label $f_q(\mathbf{x})$ and the quotient $c = \lfloor \sum_{i=1}^N x_i / q \rfloor$. Since our method required the auxiliary target $f_{Kq}(\mathbf{x})$ during training, it was dynamically generated by first reconstructing the total sum as $c \cdot q + f_q(\mathbf{x})$ and then applying the modulo $Kq$ operation.

\paragraph{Training setup.}
An encoder-only Transformer architecture consisting of 4 layers and 4 attention heads, with an embedding dimension of 256 and a feed-forward network (FFN) dimension of 2048 was adopted. Either token embedding (cf. \textbf{Approach~1}) or angular embedding (cf. \textbf{Approach~2}) was used, as detailed in \cref{sec:computational_pipeline}. A custom loss to avoid embedding collapse was also used for the sparse method as \citet{conf/icml/SaxenaAWL25} suggested.
The dropout rate was set to 0.0, and positional embeddings were not used. For optimization, the AdamW optimizer \citep{loshchilov2018decoupled} was used with $(\beta_1, \beta_2) = (0.9, 0.999)$ and a weight decay of 0.1. All models were trained for 10 epochs with a batch size of 250. The learning rate was set to $3 \times 10^{-5}$ and scheduled with a linear decay, using a warmup ratio of 0.05 rather than fixed steps to accommodate varying training sample sizes. All experiments were conducted on a single GPU of NVIDIA RTX 6000 Ada. Training time per model ranged from approximately 1 minute to 20 hours, depending on $N$, $q$, and the training sample size.

\paragraph{Hyperparameter search.}
The optimal hyperparameter pair ($K \in \{4, \dots, 9\}$, $r \in \{0.1, \dots, 0.4\}$) was determined via a grid search on a small dataset ($|\mathcal{D}_{\text{train}}| = 100\text{K}$) and applied across all sample sizes. The specific values of $K$ and $r$ for each setting are detailed in \cref{sec:hyperparameters}. This search only required observing relative performance differences rather than full convergence. \Cref{ex:ex3} provides an ablation study demonstrating strong hyperparameter robustness.

\paragraph{Evaluation.}
Performance was measured by match accuracy on $\mathcal{D}_{\text{test}}$, defined as the proportion of predictions that perfectly matched the ground-truth labels. In the case of angular embedding, the model outputs a point $(\cos\hat{\phi}, \sin\hat{\phi})$, which was first converted back into an angle $\hat{\phi} \in [0, 2\pi)$. This angle was then mapped to a continuous value $\hat{s} = \frac{\hat{\phi}q}{2\pi} \in [0, q)$, and finally rounded to the nearest integer before evaluating the match.

\begin{table*}[tbp]
    \centering
    \setlength{\tabcolsep}{4.5pt}
    \caption{Comparison of match accuracy (\%) between our method and the baseline \citep{conf/icml/SaxenaAWL25} using $100\text{K}$ and $1\text{M}$ training samples for \textbf{(a)} token embedding and \textbf{(b)} angular embedding. Note that the baseline is adapted depending on the embedding type, as detailed in the main text. Our method achieves higher match accuracy across various numbers of terms $N$ and moduli $q$.}
    \label{tab:accuracy_comparison_by_sample_size}
    
    \begin{subtable}[t]{0.48\linewidth}
        \centering
        \caption{Token embedding}
        \label{tab:accuracy_comparison_by_sample_size_token}
        \begin{tabular}{cccccc}
            \toprule
             & & \multicolumn{2}{c}{$100\text{K}$} & \multicolumn{2}{c}{$1\text{M}$} \\
            \cmidrule(lr){3-4} \cmidrule(lr){5-6}
            $N$ & $q$ & Ours & Baseline & Ours & Baseline \\
            \midrule
            8 & 31 & \textbf{89.1} & 3.2 & \textbf{99.4} & 98.6 \\
             & 97 & \textbf{13.9} & 0.9 & \textbf{90.1} & 74.9 \\
             & 257 & \textbf{1.0} & 0.5 & \textbf{34.0} & 0.5 \\
            \midrule
            16 & 31 & \textbf{58.4} & 3.3 & \textbf{98.9} & 80.7 \\
             & 97 & \textbf{6.2} & 1.0 & \textbf{81.9} & 1.1 \\
             & 257 & \textbf{0.6} & 0.4 & \textbf{21.7} & 0.4 \\
            \midrule
            32 & 31 & \textbf{33.1} & 2.9 & \textbf{96.5} & 38.7 \\
             & 97 & \textbf{1.7} & 1.1 & \textbf{62.2} & 1.0 \\
             & 257 & \textbf{0.5} & 0.4 & \textbf{12.6} & 0.3 \\
            \bottomrule
        \end{tabular}
    \end{subtable}
    \hfill %
    \begin{subtable}[t]{0.48\linewidth}
        \centering
        \caption{Angular embedding}
        \label{tab:accuracy_comparison_by_sample_size_angle}
        \begin{tabular}{cccccc}
            \toprule
             & & \multicolumn{2}{c}{$100\text{K}$} & \multicolumn{2}{c}{$1\text{M}$} \\
            \cmidrule(lr){3-4} \cmidrule(lr){5-6}
            $N$ & $q$ & Ours & Baseline & Ours & Baseline \\
            \midrule
            16 & 97 & \textbf{98.9} & 32.2 & \textbf{99.9} & 92.6 \\
             & 257 & \textbf{93.0} & 11.8 & \textbf{99.4} & 63.8 \\
             & 433 & \textbf{80.7} & 6.8 & \textbf{98.5} & 45.7 \\
            \midrule
            32 & 97 & \textbf{95.6} & 4.2 & \textbf{99.6} & 81.1 \\
             & 257 & \textbf{73.5} & 0.6 & \textbf{88.5} & 44.9 \\
             & 433 & \textbf{57.5} & 0.9 & \textbf{92.2} & 27.7 \\
            \midrule
            64 & 97 & \textbf{83.2} & 0.9 & \textbf{97.4} & 58.7 \\
             & 257 & \textbf{43.2} & 0.3 & \textbf{81.1} & 23.4 \\
             & 433 & \textbf{29.0} & 0.2 & \textbf{53.8} & 11.4 \\
            \bottomrule
        \end{tabular}
    \end{subtable}
\end{table*}

\paragraph{Results: scalability and sample efficiency.}
As shown in \cref{tab:accuracy_comparison_by_sample_size}, our method consistently outperforms the baseline across all configurations. For token embeddings (\cref{tab:accuracy_comparison_by_sample_size_token}), the baseline completely fails to learn with 100K samples, whereas our method successfully initiates learning (e.g., $89.1\%$ at $N=8, q=31$). Our approach also exhibits superior scalability: when increasing $N$ from 8 to 32 (using 1M samples, $q=31$), the baseline's accuracy collapses from $98.6\%$ to $38.7\%$, while ours degrades minimally ($99.4\% \to 96.5\%$).
Under angular embeddings (\cref{tab:accuracy_comparison_by_sample_size_angle}), our method demonstrates even extreme sample efficiency, achieving higher accuracies with just 100K samples than the baseline does with 1M. For instance, at $N=16, q=257$, our model reaches $93.0\%$ (100K) compared to the baseline's $63.8\%$ (1M), with similar trends observed even at larger scales (e.g., $N=64, q=97$).

\subsection{Larger-scale modular addition}\label{ex:ex2}
We next address larger-scale modular addition tasks: $N \in \{16, 32, 64\}$ and $q \in \{257, 3329, 42899, 974269\}$. This was particularly targeted by the sparse method in the context of attack on LWE.  
In this case, the evaluation metric ($\tau$-accuracy with $\tau \in \{0.01, 0.05\}$) was used. 
\begin{equation*}
  \tau\text{-accuracy} = \frac{1}{|\mathcal{D}_\text{test}|} \sum_{x \in \mathcal{D_\text{test}}} \mathbbm{1}_{\| \hat{y} - y\| \leq \tau q}
\end{equation*}
where $\hat{y}$ is a continuous value converted from the model's output for the input $\mathbf{x}$, and $y$ is the target label of $\mathbf{x}$. Here, $\| \hat{y} - y \|$ represents the distance calculated as $\min(|\hat{y} - y|, q - |\hat{y} - y|)$ to account for the periodic nature of modulo $q$ (i.e., $0$ and $q-1$ are near). $\mathbbm{1}_{S}$ is an indicator function of a set $S$ satisfying $\mathbbm{1}_{S}(v) = 1$ if $v \in S$ and $0$ otherwise.
This metric relaxes the strict exact-match condition by measuring the proportion of predictions that fall within a tolerance margin $\tau q$ from the true target value.

\paragraph{Results.}
The experimental results show that the proposed method maintains high sample efficiency and excellent scalability even in large-scale modular addition. As shown in  \cref{tab:tau_accuracy_comparison_by_sample_size_angle}, even in these larger-scale settings, our method consistently achieves higher $\tau$-accuracy than the baseline method given the same sample size. Specifically, under the condition of $1\text{M}$ samples and $N=128$, the Baseline method shows almost no improvement in $\tau$-accuracy, whereas our method records high $\tau$-accuracy. Furthermore, comparing the results of our method trained on $1\text{M}$ samples with those of the Baseline method trained on $10\text{M}$ samples reveals comparable high performance. Notably, in the highly complex setting where $N=128$, our method using $1\text{M}$ samples outperforms the Baseline method using $10\text{M}$ samples.

\begin{table*}[tbp]
    \centering
    \caption{Comparison of $\tau$-accuracy (\%) ($\tau \in \{0.05, 0.1\}$) between our method and the baseline~\citep{conf/icml/SaxenaAWL25} using $100\text{K}$ and $1\text{M}$ training samples for angular embedding. Our method consistently achieves higher $\tau$-accuracy even in the large-scale modular addition settings based on the prior work.}
    \label{tab:tau_accuracy_comparison_by_sample_size_angle}
    \begin{tabular}{cccccccccc}
      \toprule
      & & \multicolumn{4}{c}{$\tau=0.1$ Accuracy (\%)} & \multicolumn{4}{c}{$\tau=0.05$ Accuracy (\%)} \\
      \cmidrule(lr){3-6} \cmidrule(lr){7-10}
      & & \multicolumn{2}{c}{$1\text{M}$} & \multicolumn{2}{c}{$10\text{M}$} & \multicolumn{2}{c}{$1\text{M}$} & \multicolumn{2}{c}{$10\text{M}$} \\
      \cmidrule(lr){3-4} \cmidrule(lr){5-6} \cmidrule(lr){7-8} \cmidrule(lr){9-10}
      $N$ & $q$ & Ours & Baseline & Ours & Baseline & Ours & Baseline & Ours & Baseline \\
      \midrule
      16 & 257 & \textbf{100.0} & 98.4 & \textbf{100.0} & 100.0 & \textbf{100.0} & 92.1 & \textbf{100.0} & 99.8 \\
       & 3329 & \textbf{100.0} & 98.3 & \textbf{100.0} & 100.0 & \textbf{100.0} & 92.5 & \textbf{100.0} & 99.9 \\
       & 42899 & \textbf{100.0} & 98.6 & \textbf{100.0} & 100.0 & \textbf{99.9} & 93.3 & \textbf{100.0} & 99.9 \\
       & 974269 & \textbf{100.0} & 98.4 & \textbf{100.0} & 100.0 & \textbf{100.0} & 93.2 & \textbf{100.0} & 99.9 \\
      \midrule
      32 & 257 & \textbf{100.0} & 95.9 & \textbf{100.0} & 99.9 & \textbf{99.1} & 82.4 & \textbf{100.0} & 99.6 \\
       & 3329 & \textbf{99.9} & 95.9 & \textbf{100.0} & 99.9 & \textbf{99.2} & 82.3 & \textbf{100.0} & 99.5 \\
       & 42899 & \textbf{100.0} & 95.6 & \textbf{100.0} & 99.9 & \textbf{99.4} & 81.5 & \textbf{100.0} & 99.5 \\
       & 974269 & \textbf{99.7} & 95.2 & \textbf{100.0} & 100.0 & \textbf{97.8} & 79.9 & \textbf{100.0} & 99.5 \\
      \midrule
      64 & 257 & \textbf{99.9} & 80.2 & \textbf{100.0} & 99.8 & \textbf{98.1} & 53.3 & \textbf{100.0} & 98.8 \\
       & 3329 & \textbf{99.8} & 70.5 & \textbf{100.0} & 99.8 & \textbf{98.0} & 43.4 & \textbf{100.0} & 98.7 \\
       & 42899 & \textbf{99.7} & 62.1 & \textbf{100.0} & 99.7 & \textbf{97.6} & 36.5 & \textbf{100.0} & 98.2 \\
       & 974269 & \textbf{99.7} & 82.4 & \textbf{100.0} & 99.7 & \textbf{97.8} & 55.8 & \textbf{100.0} & 98.5 \\
      \midrule
      128 & 257 & \textbf{99.1} & 10.3 & \textbf{100.0} & 99.1 & \textbf{96.4} & 5.0 & \textbf{100.0} & 93.9 \\
       & 3329 & \textbf{99.6} & 7.6 & \textbf{100.0} & 99.0 & \textbf{97.2} & 3.9 & \textbf{100.0} & 93.1 \\
       & 42899 & \textbf{98.4} & 23.4 & \textbf{100.0} & 99.1 & \textbf{94.4} & 11.9 & \textbf{99.6} & 93.8 \\
       & 974269 & \textbf{99.4} & 18.6 & \textbf{100.0} & 99.2 & \textbf{97.0} & 9.5 & \textbf{100.0} & 93.9 \\
      \bottomrule
    \end{tabular}
\end{table*}

\subsection{Sensitivity to hyperparameters} \label{ex:ex3}
We evaluate the hyperparameter sensitivity of our approach by reporting the average, minimum, and maximum match accuracies over the grid search space ($K \in \{4, \dots, 9\}, r \in \{0.1, \dots, 0.4\}$). These results are then compared against the baseline from \cref{ex:ex1}.

\paragraph{Robustness to hyperparameter selection.}
As summarized in \cref{tab:average_accuracy_comparison}, our method demonstrates highly stable learning performance across various combinations of $K$ and $r$. The narrow gap between the minimum and maximum accuracies indicates that the model yields consistently high performance without relying on heavy parameter tuning. Furthermore, even the minimum accuracy within this search space shows a clear improvement over the baseline result from \cref{ex:ex1}. For instance, under the $N=64, q=433$ setting with angular embedding, our minimum accuracy of $34.0\%$ exceeds the baseline's $11.4\%$. Similarly, for token embedding ($N=16, q=97$), the worst-case accuracy in our grid still substantially surpasses the baseline's $1.1\%$. This confirms that our approach is inherently robust to hyperparameter choices.Detailed results across a broader range ($K \in \{2, \dots, 10\}, r \in \{0.1, \dots, 0.9\}$) are provided in \cref{sec:additional_heatmaps}.

\begin{table}[htbp]
    \centering
    \caption{Minimum, average, and maximum match accuracy (\%) of the proposed method calculated across the practical grid search space ($K \in \{4, \dots, 9\}, r \in \{0.1, \dots, 0.4\}$). For reference, the baseline result from~\cref{tab:accuracy_comparison_by_sample_size} is also provided. The improvement column reports average accuracy minus baseline (\%). Our method achieves higher match accuracy even on average, demonstrating strong hyperparameter robustness.}
    \label{tab:average_accuracy_comparison}
    \begin{tabular}{lrrrr|r}
        \toprule
        \multirow{2}{*}{Embedding} & \multicolumn{2}{c}{Problem Setting} & \multicolumn{1}{c}{Ours} & \multicolumn{1}{c}{Baseline} & \multicolumn{1}{c}{Gain} \\
        \cmidrule(lr){2-3}
        & $N$ & $q$ & Min /Avg. /Max & (from \cref{tab:accuracy_comparison_by_sample_size}) & Avg. $-$ Base. \\
        \midrule
        \multirow{9}{*}{Token}
        & 8  & 31  & 99.2~/~\textbf{99.5}~/~99.6 & 98.6 & \texttt{+}0.90 \\
        &    & 97  & 90.1~/~\textbf{93.1}~/~96.7 & 74.9 & \texttt{+}18.2 \\
        &    & 257 & 27.3~/~\textbf{33.8}~/~37.0 & 0.5 & \texttt{+}33.3 \\
        & 16 & 31  & 98.1~/~\textbf{98.9}~/~99.4 & 80.7 & \texttt{+}18.2 \\
        &    & 97  & 76.1~/~\textbf{81.1}~/~84.2 & 1.1 & \texttt{+}80.0 \\
        &    & 257 & 14.2~/~\textbf{23.9}~/~27.1 & 0.4 & \texttt{+}23.5 \\
        & 32 & 31  & 92.7~/~\textbf{97.2}~/~98.5 & 38.7 & \texttt{+}58.5 \\
        &    & 97  & 46.0~/~\textbf{63.6}~/~69.6 & 1.0 & \texttt{+}62.6 \\
        &    & 257 & 0.4~/~\textbf{13.5}~/~17.8  & 0.3 & \texttt{+}13.2 \\
        \midrule
        \multirow{9}{*}{Angular}
        & 16 & 97  & 99.7~/~\textbf{99.9}~/~100 & 92.6 & \texttt{+}7.30 \\
        &    & 257 & 97.4~/~\textbf{99.1}~/~99.8  & 63.8 & \texttt{+}35.3 \\
        &    & 433 & 90.2~/~\textbf{96.6}~/~99.4   & 45.7 & \texttt{+}50.9 \\
        & 32 & 97  & 98.0~/~\textbf{99.3}~/~99.9  & 81.1 & \texttt{+}18.2 \\
        &    & 257 & 84.3~/~\textbf{93.5}~/~99.0   & 44.9 & \texttt{+}48.6 \\
        &    & 433 & 61.8~/~\textbf{83.8}~/~94.4   & 27.7 & \texttt{+}56.1 \\
        & 64 & 97  & 91.8~/~\textbf{96.5}~/~99.5   & 58.7 & \texttt{+}37.8 \\
        &    & 257 & 64.7~/~\textbf{80.5}~/~94.1   & 23.4 & \texttt{+}57.1 \\
        &    & 433 & 34.0~/~\textbf{54.7}~/~78.4   & 11.4 & \texttt{+}43.3 \\
        \bottomrule
    \end{tabular}
\end{table}

\begin{wrapfigure}{R}{0.5\textwidth}
    \vspace{-1.25\baselineskip}
        \includegraphics[width=1.\linewidth]{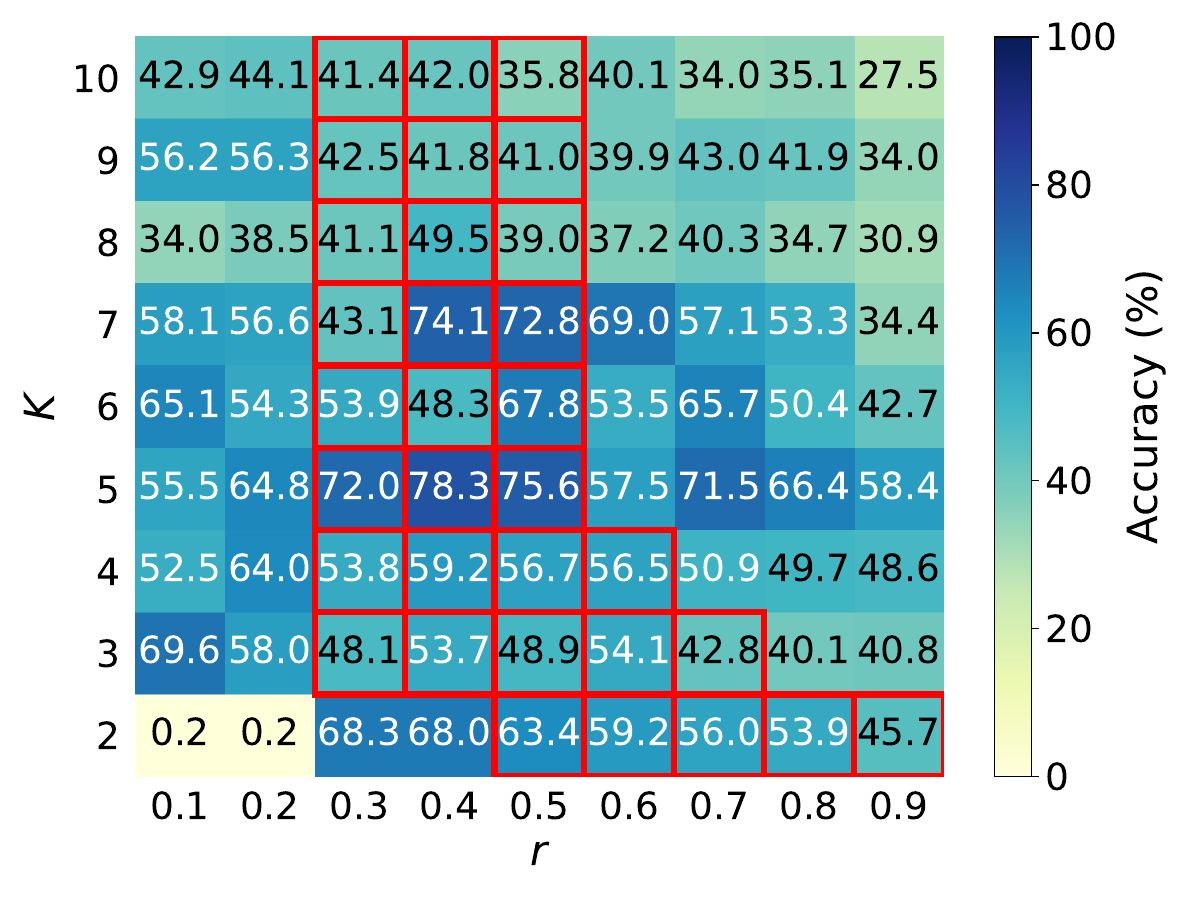}\\[3pt]
        \captionsetup{skip=-15pt}%
        \caption{Heatmap of match accuracy for modular addition with angular embeddings ($N=64, q=433$). The red box highlights the region where the ratio $\rho$ in \cref{eq:wrap_scale_factor_ratio} is between $0.8$ and $1.2$, corresponding to the theoretical difficulty of the prior work shown in \cref{fig:wrap_scale_factor_heatmap}.}
        \label{fig:red}
    \vspace{-1.\baselineskip}
   \end{wrapfigure}
   
\paragraph{Relationship between accuracy and difficulty.}
While the proposed method is robust on average, the accuracy can fluctuate significantly depending on the configuration. \cref{fig:red} presents a heatmap of a setting with such high variance to examine the underlying trends. The red box highlights the region where the ratio in \cref{eq:wrap_scale_factor_ratio}--defined as a ratio of expected numbers of wraps around the modulus by the proposed and sparse methods--lies between $0.8$ and $1.2$. As shown in \cref{subsec:analysis}, in the case when the ratio equals $1.0$, the difficulty of a problem under the proposed method theoretically coincides with that under the sparse method. Although high accuracy is not exclusively confined to this box, and the theoretical range does not guarantee uniformly high performance, the region where the ratio approaches $1.0$ encompasses the optimal combinations of $K$ and $r$. The region where the ratio approaches $1.0$ encompasses the high-performing combinations of $K$ and $r$.  Indeed, relatively high-accuracy regions, such as $K \in \{5, \dots, 7\}$ and $r \in \{0.4, 0.5\}$, largely overlap with this theoretical range.

\section{Conclusion}\label{ch:conclusion}
We proposed a novel training method for stable modular addition learning by introducing an expanded modulus $Kq$ as an auxiliary task alongside the target modulus $q$. Unlike the sparse method \citep{conf/icml/SaxenaAWL25}, which changes the input distribution during training, our approach keeps the input distribution unchanged and thus avoids this source of covariate shift.
Theoretical analysis showed that our method reduces problem difficulty similarly to sparse training, but without modifying the input distribution.

Experimental results demonstrated that our method improves learning efficiency and accuracy across various conditions of terms $N$, moduli $q$, and data sizes. Notably, our method achieves high accuracy even under conditions of large $N$, large $q$, and small data sizes, where baseline methods fail to learn. Furthermore, it achieves equal or superior match accuracy and $\tau$-accuracy even when the training data size is reduced to one-tenth. These results indicate that methods avoiding distribution shifts possess high learning efficiency and scalability.

\paragraph{Limitations and future work.} Our method is limited to modular addition tasks. In future work, we aim to extend this method to other operations, such as modular multiplication. Additionally, establishing a more efficient optimization method for the hyperparameters $K$ and $r$, which are currently determined by grid search, remains an important challenge.  Investigating the relationship between the theoretical difficulty ratio in \cref{eq:wrap_scale_factor_ratio} and empirical performance is left for future work. Based on this relationship, we aim to develop a principled search strategy that constrains the parameter space. We will also conduct a detailed analysis of how the problem difficulty changes.

\section*{Acknowledgments}
This research was partially supported by JST PRESTO Grant Number JPMJPR24K, JST BOOST Program Grant Number JPMJBY24C6, and JSPS Program for Forming Japan’s Peak Research Universities (J-PEAKS) Grant Number JPJS00420230002, Mitsubishi Electric Information Technology R\&D Center, and JSPS KAKENHI Grant Number 26K02996.

\newpage
\bibliographystyle{plainnat}
\bibliography{bib}

\clearpage
\appendix
\section{Formal restatement and proof of \cref{prop:gap-sparse}}\label{sec:proof-gap-sparse}

\Cref{prop:gap-sparse} in the main text is informal in two respects: it leaves the loss 
unspecified and identifies ``the expected loss on a uniformly random zero-free input'' without 
fixing a measure. The formal restatement below makes both precise; \cref{tab:tv-numerical} 
reports the numerical value of \cref{eq:gap-tv} for parameter ranges typical of 
modular-addition experiments.

\begin{proposition}[Formal version of \cref{prop:gap-sparse}]\label{prop:gap-sparse-formal}
Let \(\mathcal{X}:=\{0,\dots,q-1\}^N\) and \(n_0(x):=\#\{i\in[N]:x_i=0\}\). 
Suppose \(P_{\mathrm{te}}\) is uniform on \(\mathcal{X}\), and \(P_{\mathrm{tr}}\) is the sparse 
construction of \mbox{\citet{conf/icml/SaxenaAWL25}}: sample \(z\sim g\) with 
\(g(z)\propto 1/\sqrt{N-z+1}\) for \(z\in\{1,\dots,N\}\), choose \(z\) positions uniformly at 
random, fill them independently and uniformly from \(\{0,\dots,q-1\}\), and set the remaining 
\(N-z\) positions to zero. 
Let \(L_f(x)\ge 0\) be a non-negative pointwise loss, and let 
\(\bar L_f(0):=\mathbb{E}_{x\sim\mathrm{Unif}(\{x:n_0(x)=0\})}[L_f(x)]\) be the loss averaged 
over a uniformly random zero-free input. 
If \(\bar L_f(0)\ge\varepsilon\) and \(R_{\mathrm{tr}}(f)\le\delta\), then
\begin{align*}
    R_{\mathrm{te}}(f)-R_{\mathrm{tr}}(f)
    \;\ge\;
    \varepsilon\!\left(1-\tfrac{1}{q}\right)^{\!N}-\delta.
\end{align*}
\end{proposition}
\begin{proof}
\(P_{\mathrm{te}}\) is permutation-invariant in the coordinates of \(x\), so conditional on 
\(n_0(x)=0\) it is uniform on the set \(\{x:n_0(x)=0\}\) of zero-free inputs. 
Therefore \(\mathbb{E}_{P_{\mathrm{te}}}[L_f(x)\mid n_0(x)=0]=\bar L_f(0)\). 
Decomposing \(R_{\mathrm{te}}(f)\) by conditioning on \(n_0(x)\) and dropping the non-negative 
contributions from \(n_0(x)\ne 0\),
\begin{align*}
    R_{\mathrm{te}}(f)
    \;=\;
    \sum_{k=0}^{N} P_{\mathrm{te}}(n_0(x)=k)\,\mathbb{E}_{P_{\mathrm{te}}}\!\left[L_f(x)\mid n_0(x)=k\right]
    \;\ge\;
    P_{\mathrm{te}}(n_0(x)=0)\,\bar L_f(0)
    \;\ge\;
    \left(1-\tfrac{1}{q}\right)^{\!N}\!\varepsilon,
\end{align*}
where the last step uses 
\(P_{\mathrm{te}}(n_0(x)=0)=\Pr_{Z\sim\mathrm{Bin}(N,1/q)}[Z=0]=(1-1/q)^N\) (since under 
the uniform distribution each entry independently equals \(0\) with probability \(1/q\)) and 
\(\bar L_f(0)\ge\varepsilon\). 
Subtracting \(R_{\mathrm{tr}}(f)\le\delta\) yields the bound.
\end{proof}

\begin{table}[th]
\centering
\caption{Numerical lower bound 
\(\mathrm{LB}(q,N):=(1-1/q)^N\) (the prefactor in \cref{eq:gap-tv}) 
for parameter ranges typical of modular-addition experiments. The reported value bounds 
\(R_{\mathrm{te}}(f)-R_{\mathrm{tr}}(f)\) from below by \(\varepsilon\cdot\mathrm{LB}(q,N)-\delta\).}
\label{tab:tv-numerical}
\begin{tabular}{c|cccc}
\toprule
$N\backslash q$ & 17 & 23 & 47 & 113 \\
\midrule
8  & 0.61 & 0.70 & 0.84 & 0.93 \\
16 & 0.38 & 0.49 & 0.71 & 0.87 \\
24 & 0.23 & 0.34 & 0.60 & 0.81 \\
32 & 0.14 & 0.24 & 0.50 & 0.75 \\
\bottomrule
\end{tabular}
\end{table}

\section{Robustness to Model Architecture Variations}
\label{app:model_stability}
We experimentally observed that the performance of the sparse method in \citet{conf/icml/SaxenaAWL25} is highly sensitive to the underlying model architecture. To systematically investigate the extent of this instability, we evaluated the match accuracy across 16 different model configurations. Our results demonstrate that the baseline method suffers from severe performance fluctuations, heavily depending on specific architectural choices. 

Following the base experimental setup of \cref{ex:ex1}.
These 16 configurations were generated by exhaustively combining four key architectural hyperparameters: the position of the normalization layer, the inclusion of bias terms, the weight initialization strategy, and the dropout rate. The specific settings for each component are detailed in \cref{tab:model_hyperparameters}. The experiment was conducted using token embedding with a sequence length of $N=8$ and a modulus of $q=97$.

\Cref{tab:model_stability_results} presents the match accuracy across all 16 configurations. The results clearly illustrate that the sparse method in \citet{conf/icml/SaxenaAWL25} is highly sensitive to minor architectural choices, frequently suffering from severe performance degradation. For example, under the Pre-Norm and Default initialization setting without dropout, simply removing the bias terms causes the baseline's accuracy to plummet drastically from $74.9\%$ to a mere $1.2\%$. Such extreme fragility to standard model variations serves as a strong motivation for exploring a new, more robust approach. While our proposed method also exhibits some performance variations depending on the configuration (such as a general decrease in accuracy when dropout is applied), the magnitude of these drops is significantly smaller. In the same example of removing the bias terms, our method maintains a consistently high accuracy (shifting only from $90.1\%$ to $91.5\%$), thereby confirming its superior structural stability.

\begin{table}[htbp]
  \centering
  \caption{Details of the hyperparameter choices for the model architecture variations: normalization layer, bias term, weight initialization, and dropout.}
  \label{tab:model_hyperparameters}
  \begin{tabularx}{0.95\linewidth}{llX}
    \toprule
    Component & Value & Description \\
    \midrule
    \multirow{2}{*}{Normalization} 
        & Pre-Norm & Layer normalization is applied before each sub-layer. \\
        & Post-Norm & Layer normalization is applied after each sub-layer. \\
    \addlinespace
    \multirow{2}{*}{Bias Term} 
        & Incl. & Bias terms are included in all learnable layers. \\
        & Excl. & Bias terms are removed from all learnable layers. \\
    \addlinespace
    \multirow{3}{*}{Weight Initialization} 
        & $\sigma=0.02$ & Linear and embedding layers are initialized from a normal distribution $\mathcal{N}(0, 0.02^2)$. \\
        & Default & Embedding layers are initialized from $\mathcal{N}(0, 1)$, and linear layers use Kaiming initialization \citep{DBLP:conf/iccv/HeZRS15}. \\
    \addlinespace
    \multirow{2}{*}{Dropout} 
        & 0.0 & No dropout is applied. \\
        & 0.1 & Dropout rate is set to 0.1. \\
    \bottomrule
  \end{tabularx}
\end{table}

\begin{table}[htbp]
  \centering
  \caption{Match accuracy (\%) across 16 model configurations for token embedding ($N=8, q=97$). The sparse method in \citet{conf/icml/SaxenaAWL25} suffers from severe instability, whereas the proposed method consistently achieves high match accuracy regardless of the architecture.}
  \label{tab:model_stability_results}
  \begin{tabular}{llll cc}
    \toprule
    \multicolumn{4}{c}{Model Configuration} & \multicolumn{2}{c}{Match Accuracy (\%)} \\
    \cmidrule(lr){1-4} \cmidrule(lr){5-6}
    Normalization & Bias & Initialization & Dropout & Sparse\citep{conf/icml/SaxenaAWL25} & Ours \\
    \midrule
    \multirow{8}{*}{Pre-Norm} 
        & \multirow{4}{*}{Incl.} 
            & \multirow{2}{*}{$\sigma=0.02$} 
                & 0.0 & 5.1 & 73.5 \\
            & & & 0.1 & 10.2 & 33.9 \\
            \cmidrule{3-6}
            & & \multirow{2}{*}{Default} 
                & 0.0 & 74.9 & 90.1 \\
            & & & 0.1 & 97.2 & 17.5 \\
        \cmidrule{2-6}
        & \multirow{4}{*}{Excl.} 
            & \multirow{2}{*}{$\sigma=0.02$} 
                & 0.0 & 1.5 & 77.7 \\
            & & & 0.1 & 4.8 & 33.8 \\
            \cmidrule{3-6}
            & & \multirow{2}{*}{Default} 
                & 0.0 & 1.2 & 91.5 \\
            & & & 0.1 & 98.5 & 14.2 \\
    \midrule
    \multirow{8}{*}{Post-Norm} 
        & \multirow{4}{*}{Incl.} 
            & \multirow{2}{*}{$\sigma=0.02$} 
                & 0.0 & 5.7 & 75.8 \\
            & & & 0.1 & 25.3 & 35.2 \\
            \cmidrule{3-6}
            & & \multirow{2}{*}{Default} 
                & 0.0 & 6.2 & 92.5 \\
            & & & 0.1 & 18.6 & 33.3 \\
        \cmidrule{2-6}
        & \multirow{4}{*}{Excl.} 
            & \multirow{2}{*}{$\sigma=0.02$} 
                & 0.0 & 1.1 & 74.0 \\
            & & & 0.1 & 0.9 & 35.0 \\
            \cmidrule{3-6}
            & & \multirow{2}{*}{Default} 
                & 0.0 & 7.2 & 90.1 \\
            & & & 0.1 & 23.8 & 27.4 \\
    \bottomrule
  \end{tabular}
\end{table}

\section{A proof of \cref{prop:expect_num_wraps} and more analysis of \cref{eq:wrap_function}}\label{sec:detail_analysis}

    In this section, we give a proof of \cref{prop:expect_num_wraps} , and describe properties of a random variable in \cref{eq:wrap_function}, which represents the number of wraps around the modulus in the proposed methods. 
    We introduce the notations used throughout this section.
    Let $N$, $K$ and $q$ be positive integers. Let $\bx = [x_1, \ldots, x_N]$ be a vector of $N$ i.i.d. random variables with a discrete distribution $U_q(x)=1/q$ if $x\in \{0, \ldots, q-1\}$ otherwise $0$. Put $S_M=\sum_{i=1}^M x_i$ for $0\leq M\leq N$.

    \subsection{A proof of \cref{prop:expect_num_wraps}}

    Let $z$ be a random variable that represents the number of non-zero entries of an input data, and has a distribution $f_{\mathrm{inv\_sqrt}}(z) \propto 1/\sqrt{N - z + 1}$ introduced in \citet{conf/icml/SaxenaAWL25}; explicitly, $g(z)=1/(C_N\sqrt{N - z + 1})$ where $C_N = \sum_{z=0}^{N} 1/\sqrt{N - z + 1}$. We define three random variables $X_0$, $X_1$ and $X_2$ as follows:
   \begin{align*}
        X_0 &= \frac{S_N}{q}, &
        X_1 &= (1-r)\frac{S_N}{q} + r\frac{S_N}{Kq}, &
        X_2 &= \frac{S_z}{q}. &
   \end{align*}  

   \begin{proposition}[\cref{prop:expect_num_wraps}]\label{prop:expect_num_warps_N}
    Keep the above notation. For $X_0$ and $X_1$, we have
    \begin{align*}
        \bE[X_0] &= \frac{N(q-1)}{2q}, &
        \bE[X_1] &= \left((1-r) + r\frac{1}{K}\right)\frac{N(q-1)}{2q} =\left((1-r) + r\frac{1}{K}\right)\bE[X_0].
    \end{align*}
    For $X_2$, when $N$ is large enough, we have
    \begin{align*}
        \bE[X_2] &\approx \frac{2}{3}\frac{N(q-1)}{2q}=\frac{2}{3}\bE[X_0].
    \end{align*}
   \end{proposition}

   \begin{proof}
       Since $x_i$'s are i.i.d. satisfying $\bE[x_i]= \sum_{j=0}^{q-1} j/q = (q-1)/2$, we have $\bE[S_N] = N(q-1)/2$. Hence, the expectations of $X_0$, $X_1$ and $X_2$ are given by
       \begin{align*}
           \bE[X_0] &= \bE\left[\frac{S_N}{q}\right] = \frac{N(q-1)}{2q}, \\
           \bE[X_1] &= \bE\left[(1-r)\frac{S_N}{q} + r\frac{S_N}{Kq}\right] = \left((1-r) + r\frac{1}{K}\right)\frac{N(q-1)}{2q}, \\
           \bE[X_2] 
           &= \sum_{z=0}^{N} \frac{1}{C_N} \frac{1}{\sqrt{N - z + 1}} \bE\left[\frac{S_z}{q}\right]
           = \sum_{z=0}^{N} \frac{1}{C_N} \frac{1}{\sqrt{N - z + 1}} \frac{z(q-1)}{2q}\\
           & = \frac{(q-1)}{2q}\sum_{z=0}^{N} \frac{z}{C_N\sqrt{N - z + 1}}
           = \frac{(q-1)}{2q} 
           \frac{\sum_{z=0}^{N}\frac{z}{\sqrt{N - z + 1}}}{\sum_{z=0}^{N}\frac{1}{\sqrt{N - z + 1}}}.
       \end{align*}
       Hence, the statement holds for $X_0$ and $X_1$. We prove the statement for $X_2$ by showing that $\bE[X_2] \approx 2\bE[X_0]/3$ when $N$ is large enough.

       If $N$ is large enough, then the followings hold:
       \begin{align*}
           \sum_{z=0}^{N}\frac{z}{\sqrt{N - z + 1}} 
            &= \sum_{k=1}^{N+1} \frac{N-k+1}{\sqrt{k}}
            = (N+1)\sum_{k=1}^{N+1} \frac{1}{\sqrt{k}} - \sum_{k=1}^{N+1} \sqrt{k}\\
            &= (N+1)\sqrt{N}\sum_{k=1}^{N+1} \frac{1}{\sqrt{k/N}}\frac{1}{N} - N\sqrt{N}\sum_{k=1}^{N+1} \sqrt{k/N}\frac{1}{N}\\
            &\approx (N+1)\sqrt{N}\int_{0}^{1} \frac{1}{\sqrt{x}} dx - N\sqrt{N}\int_{0}^{1}  \sqrt{x} dx\\
            &= 2(N+1)\sqrt{N} - \frac{2}{3}N\sqrt{N}\\
            &= \frac{4}{3}N\sqrt{N} + 2\sqrt{N}.\\
           \sum_{z=0}^{N}\frac{1}{\sqrt{N - z + 1}} 
            &= \sum_{k=1}^{N+1} \frac{1}{\sqrt{k}}\\
            &= \sqrt{N}\sum_{k=1}^{N+1} \frac{1}{\sqrt{k/N}}\frac{1}{N}\\
            &\approx \sqrt{N}\int_{0}^{1} \frac{1}{\sqrt{x}} dx\\
            &= 2\sqrt{N}
       \end{align*}
       Hence, we have
       \begin{align*}
           \bE[X_2] \approx \frac{q-1}{2q} \frac{\frac{4}{3}N\sqrt{N} + 2\sqrt{N}}{2\sqrt{N}} 
           = \frac{q-1}{2q} \left(\frac{2N}{3} + 1\right)
           \approx \frac{N(q-1)}{3q}.
       \end{align*}
   \end{proof}

    \subsection{Properties of a random variable in \cref{eq:wrap_function}}
     We describe properties of a random variable, $D_{Kq}(\bx)$, which represents the number of wraps around the modulus.
     We recall the definition.
     \begin{align}
            D_{Kq}(\bx) = \left\lfloor\frac{S_N}{Kq}\right\rfloor,
        \end{align}
     
     We begin with the following proposition on the distribution of $S_N$:
    \begin{proposition}\label{prop:distribution_of_S_N}
        Keep the above notation. The distribution of $S_N$ is defined by
        \begin{align*}
            U_q^{*N}(S_N = s) &= 
            \frac{1}{q^N}\sum_{k=0}^{\left\lfloor \frac{s}{q} \right\rfloor}(-1)^k \binom{N}{k}\binom{s-qk+N-1}{N-1},
        \end{align*}
        where $U_q^{*N}$ is the convolution of $N$ copies $U_q$.
    \end{proposition}

    \begin{proof}
        Let $G_q(z)$ be the probability generating function of $x_i$'s and let $H_{q,N}(z)$ be the probability generating function of $S_N$. Since $x_i$'s are i.i.d. and $G_q(z)=\sum_{k=0}^{q-1} z^k / q$, for $z\in (-1,1)$ we have
        \begin{align*}
            H_{q,N}(z)
            &= G_q(z)^N\\
            &= \frac{1}{q^N}\left(\sum_{k=0}^{q-1} z^k\right)^N\\
            &= \frac{1}{q^N}\left(\frac{1-z^q}{1-z}\right)^N
        \end{align*}
        On the open interval $(-1,1)$, $1/(1-z)$ is equal to the power series $\sum_{k=0}^{\infty} z^k$ because its radius of convergence is $1$. Therefore, we have
        \begin{align*}
            H_{q,N}(z)
            &= \frac{1}{q^N}\left(\frac{1-z^q}{1-z}\right)^N\\
            &= \frac{1}{q^N}\left((1-z^q)\sum_{k=0}^{\infty} z^k\right)^N\\
            &= \frac{1}{q^N}\left(1-z^q\right)^N \left(\sum_{k=0}^{\infty} z^k\right)^N\\
            &= \frac{1}{q^N}\left(\sum_{l=0}^N (-1)^l \binom{N}{l} z^{lq}\right) \left(\sum_{k=0}^{\infty} \binom{k+N-1}{N-1} z^k\right).\\
        \end{align*}
        Hence, the coefficient of $z^s$ ($s\geq 0$) in $H_{q,N}(z)$ is given by
        \begin{align*}
            \frac{1}{q^N}\sum_{\substack{k,l\geq 0\\ lq+k=s}}^N (-1)^l \binom{N}{l}\binom{k+N-1}{N-1}
            &=\frac{1}{q^N}\sum_{\substack{\frac{s}{q}\geq l \geq 0\\ k=s-lq}}^N (-1)^l \binom{N}{l}\binom{k+N-1}{N-1}.\\
            &= \frac{1}{q^N}\sum_{l=0}^{\left\lfloor \frac{s}{q} \right\rfloor}(-1)^l \binom{N}{l}\binom{s-ql+N-1}{N-1}.
        \end{align*}

    \end{proof}
    Define a random variable $D_{Kq}(\bx)$ as in \cref{eq:wrap_function}. Then, the below proposition holds:
    \begin{proposition}\label{prop:max_min_expecet}
        We have $\max\{0, N(q-1)/(2qK)-1\} \leq \bE[D_{Kq}] \leq N(q-1)/(2qK)$.
        In particular, for $r\in [0,1]$, 
        \begin{align*}
            \bE[r\cdot D_q + (1-r)\cdot D_{Kq}] \leq \frac{N(q-1)}{2q}\left(r + (1-r)\frac{1}{K}\right).
        \end{align*}
    \end{proposition}
    \begin{proof}
        By $\bE[S_N/(qK)] = N(q-1)/(2qK)$ and $\max\{0, S_N/(qK)-1\}\leq D_{Kq}\leq S_N/(qK)$, one can prove the statement.
    \end{proof}

    \begin{proposition}\label{prop:distributiuon_wraps}
        Keep the above notation. Let $P$ be the distribution of $D_{Kq}$. Then, we have
        \begin{align*}
            P(D_{Kq} = 0) = \frac{1}{q^N}\sum_{i=0}^{K-1}(-1)^i\binom{N}{i}\binom{(K-i)q+N-1}{N-1}.
        \end{align*}
    \end{proposition}
    \begin{proof}
        Since $D_{Kq}=0$ if and only if $0\leq S_N \leq Kq-1$, by \cref{prop:distribution_of_S_N} it is enough to compute 
        \begin{align}\label{eq:sum_of_distribution_of_S_N}
            \sum_{s=0}^{Kq-1} \frac{1}{q^N}\sum_{i=0}^{\left\lfloor \frac{s}{q} \right\rfloor}(-1)^i \binom{N}{i}\binom{s-qi+N-1}{N-1}.
        \end{align}

        We rewrite the above in terms of $i$ as follows:
        \begin{align*}
            \binom{N}{i}\sum_{s=iq}^{Kq-1} \binom{s-iq+N-1}{N-1}
            &= \binom{N}{i}\sum_{s=0}^{(K-i)q-1} \binom{s+N-1}{N-1}\\
            &= \binom{N}{i}\binom{(K-i)q+N-1}{N-1}\\
        \end{align*}
        Here, for the last equality, we use 
        \begin{align*}
            \binom{r}{r} + \binom{r+1}{r} + \binom{r+2}{r} + \cdots + \binom{n}{r} = \binom{n+1}{r+1},
        \end{align*}
        where $n\geq r\geq 0$.
    \end{proof}

\section{Optimal Hyperparameters Used in the Experiments}
\label{sec:hyperparameters}

\Cref{tab:optimal_hyperparameters} summarizes the specific hyperparameter pairs $(K, r)$ determined via the preliminary grid search and utilized in our experiments. As observed in the table, there is no obvious regularity or monotonic trend in the optimal values of $K$ and $r$ with respect to the sequence length $N$ or the modulus $q$. 
However, as demonstrated by the robustness analysis in \cref{ex:ex3}, our proposed method achieves consistently high match accuracy across a broad region of the hyperparameter space. Specifically, as long as $K$ is sufficiently large and $r$ is not excessively large, the model maintains stable performance. Therefore, the lack of a strict pattern in these selected optimal values does not pose a critical issue for the general applicability of our method. Nevertheless, further theoretical and empirical analyses are required as future work to systematically derive the optimal $K$ and $r$ for any given problem setting.

\begin{table}[htpb] %
    \centering
    \small
    \setlength{\tabcolsep}{4pt}
    
    \caption{Optimal hyperparameters $(K, r)$ used for each setting, determined via grid search on a $100\text{K}$ subset. The values are presented separately for \textbf{(a)} token embedding and \textbf{(b)} angular embedding. Hyphens (-) indicate combinations that were not evaluated.}
    \label{tab:optimal_hyperparameters}
    
    \begin{subtable}[t]{0.35\linewidth}
        \centering
        \caption{Token embedding}
        \label{tab:optimal_hyperparameters_token}
        \begin{tabular}{cccc}
            \toprule
            $N\backslash q$ & $31$ & $97$ & $257$ \\
            \midrule
            8   & (5, 0.2) & (4, 0.4) & (8, 0.4) \\
            16  & (5, 0.4) & (4, 0.4) & (4, 0.1) \\
            32  & (5, 0.4) & (8, 0.4) & (9, 0.1) \\
            \bottomrule
        \end{tabular}
    \end{subtable}%
    \hfill
    \begin{subtable}[t]{0.62\linewidth}
        \centering
        \caption{Angular embedding}
        \label{tab:optimal_hyperparameters_angular}
        \begin{tabular}{ccccccc}
            \toprule
            $N\backslash q$ & $97$ & $257$ & $433$ & $3329$ & $42899$ & $974269$ \\
            \midrule
            16  & (5, 0.4) & (6, 0.3) & (6, 0.3) & - & - & - \\
            32  & (4, 0.1) & (4, 0.2) & (6, 0.1) & (4, 0.4) & (4, 0.3) & (4, 0.2) \\
            64  & (4, 0.3) & (4, 0.3) & (4, 0.3) & (4, 0.3) & (4, 0.3) & (4, 0.2) \\
            128  & - & (8, 0.4) & - & (6, 0.4) & (8, 0.4) & (8, 0.3) \\
            \bottomrule
        \end{tabular}
    \end{subtable}
\end{table}

\section{Extended Heatmaps for Hyperparameter Sensitivity}
\label{sec:additional_heatmaps}

\begin{figure}[htbp]
    \centering
    
    \begin{subfigure}[b]{0.32\linewidth}
        \centering
        \HeatmapInclude[width=\linewidth, trim={0.5cm 0.5cm 0.5cm 0.5cm}, clip]{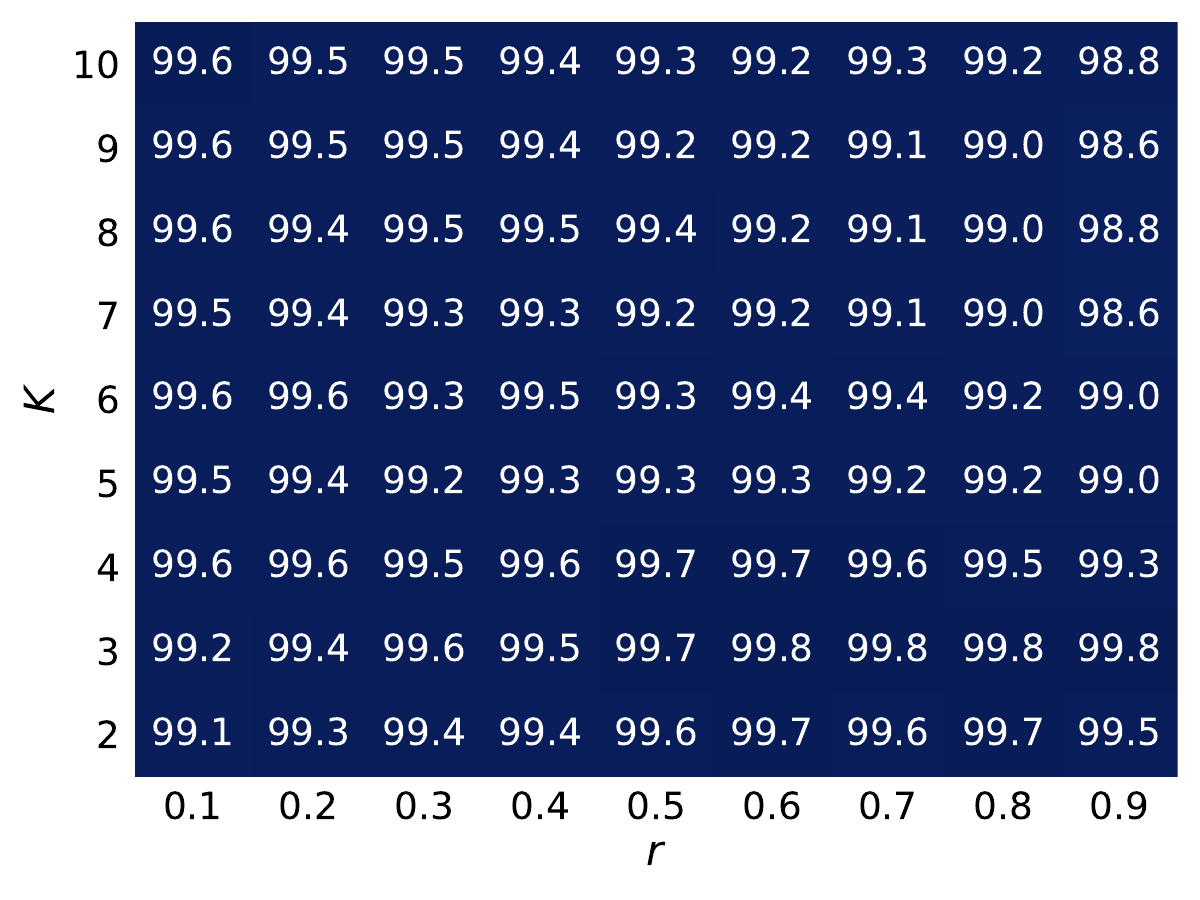}
        \caption{$N=8, q=31$}
    \end{subfigure}
    \hfill
    \begin{subfigure}[b]{0.32\linewidth}
        \centering
        \HeatmapInclude[width=\linewidth, trim={0.5cm 0.5cm 0.5cm 0.5cm}, clip]{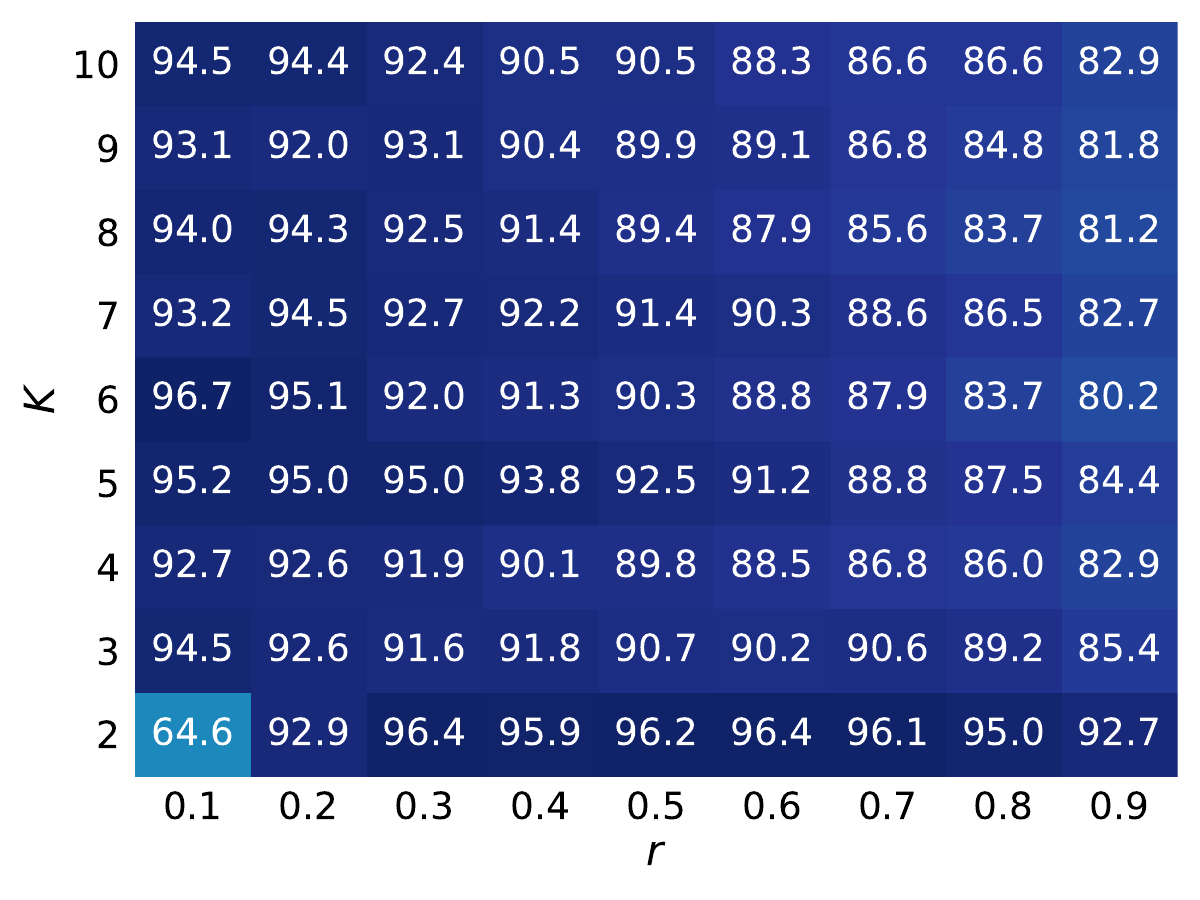}
        \caption{$N=8, q=97$}
    \end{subfigure}
    \hfill
    \begin{subfigure}[b]{0.32\linewidth}
        \centering
        \HeatmapInclude[width=\linewidth, trim={0.5cm 0.5cm 0.5cm 0.5cm}, clip]{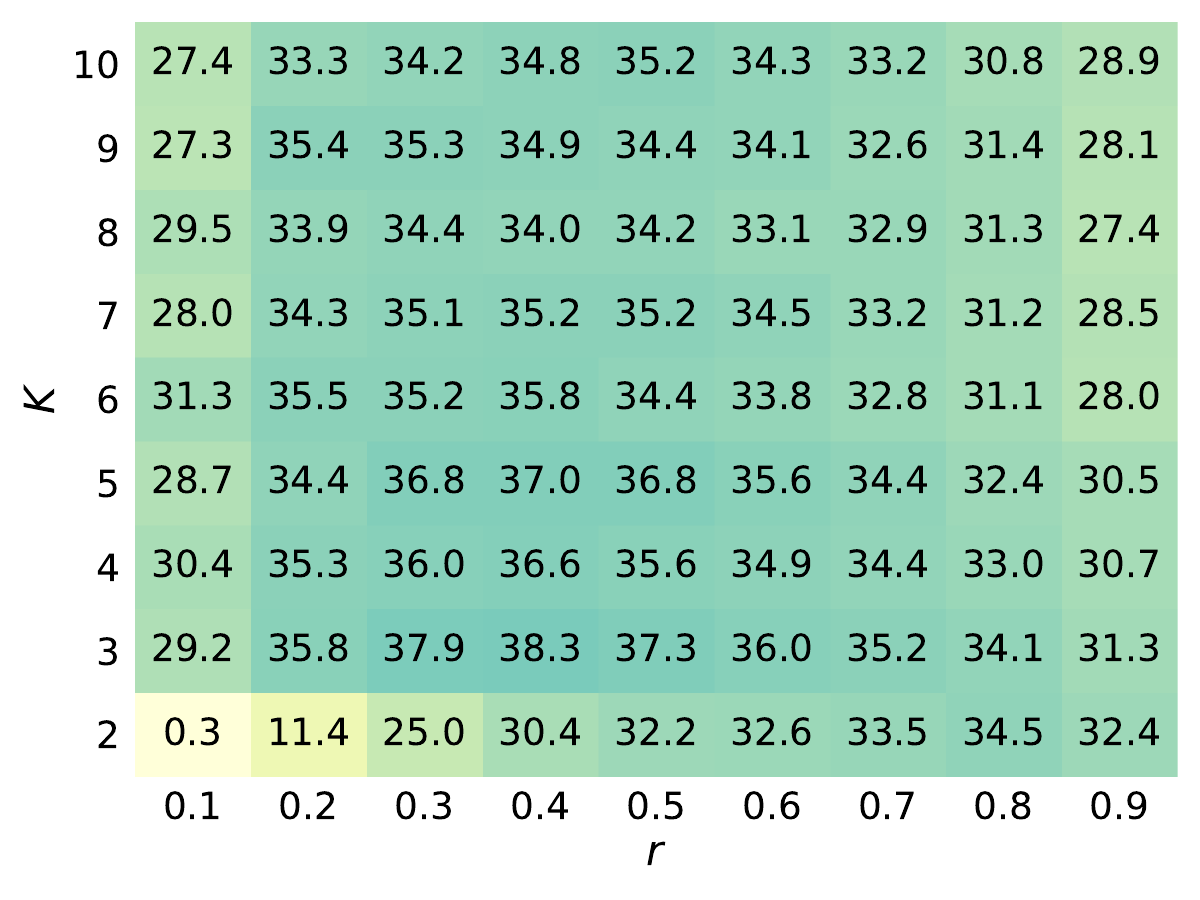}
        \caption{$N=8, q=257$}
    \end{subfigure}
    
    \vspace{1em} %
    
    \begin{subfigure}[b]{0.32\linewidth}
        \centering
        \HeatmapInclude[width=\linewidth, trim={0.5cm 0.5cm 0.5cm 0.5cm}, clip]{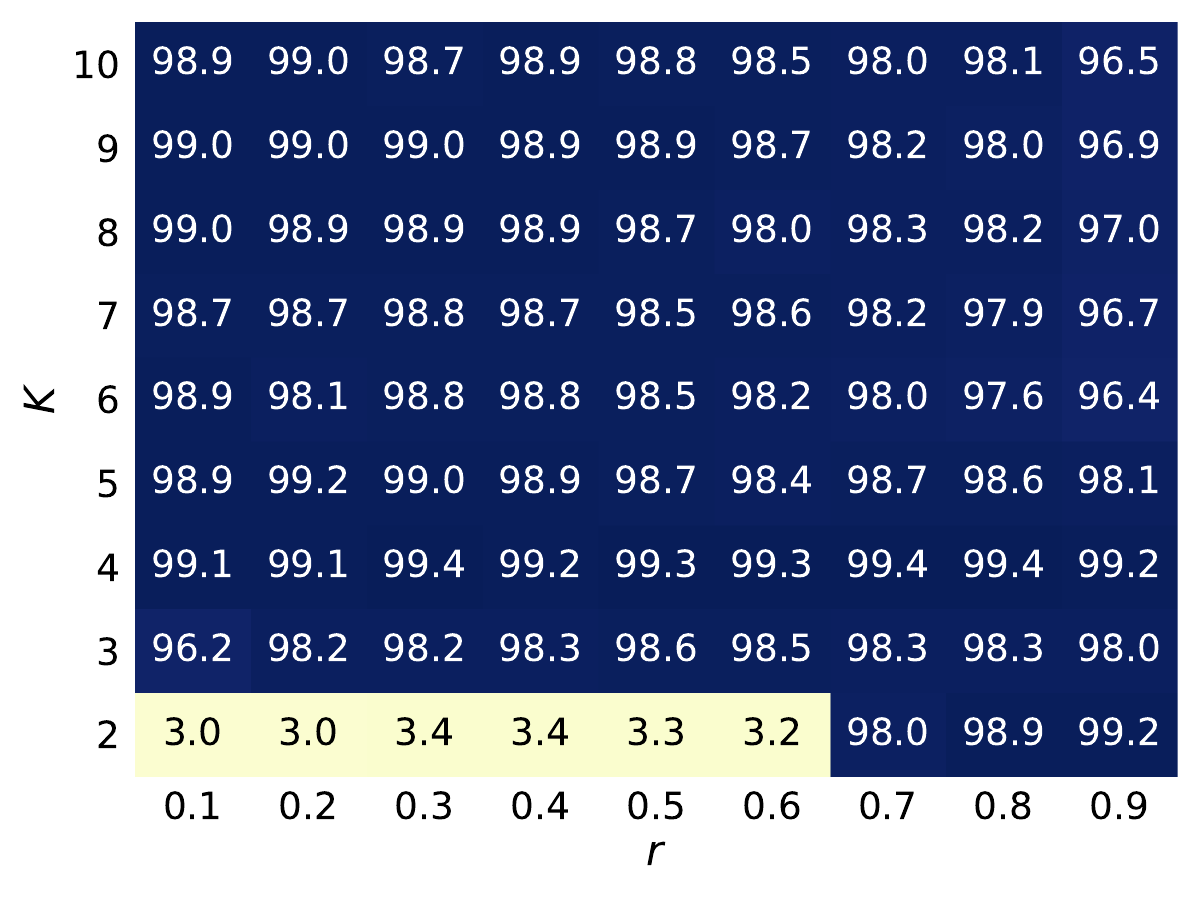}
        \caption{$N=16, q=31$}
    \end{subfigure}
    \hfill
    \begin{subfigure}[b]{0.32\linewidth}
        \centering
        \HeatmapInclude[width=\linewidth, trim={0.5cm 0.5cm 0.5cm 0.5cm}, clip]{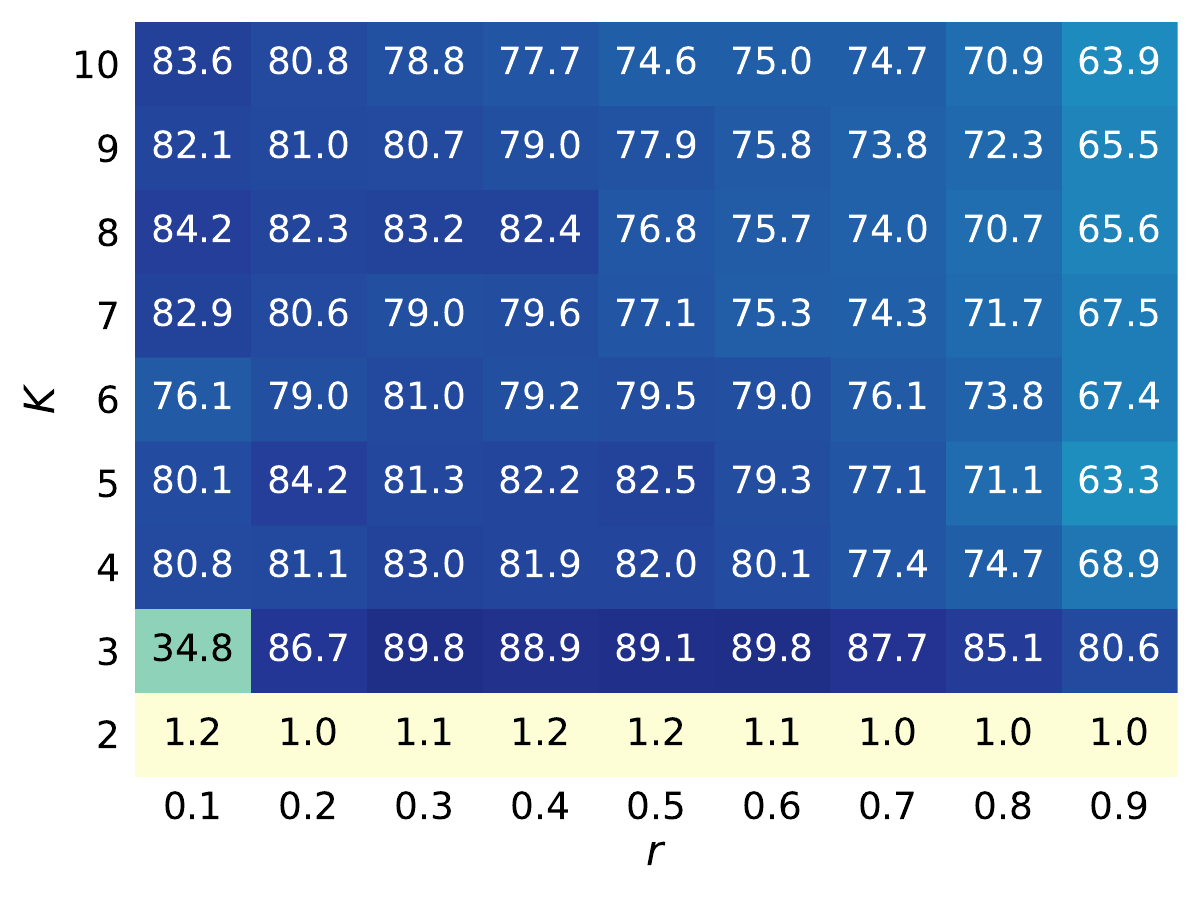}
        \caption{$N=16, q=97$}
    \end{subfigure}
    \hfill
    \begin{subfigure}[b]{0.32\linewidth}
        \centering
        \HeatmapInclude[width=\linewidth, trim={0.5cm 0.5cm 0.5cm 0.5cm}, clip]{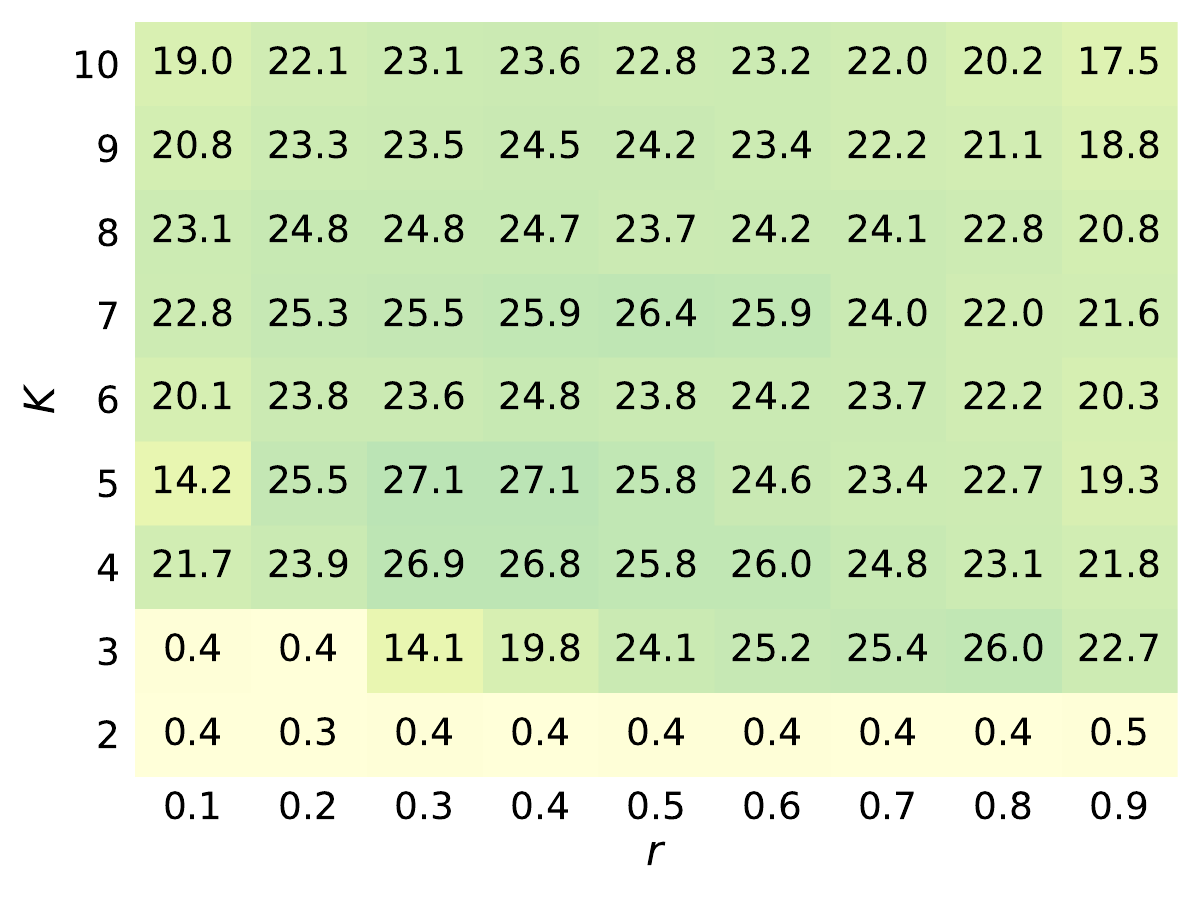}
        \caption{$N=16, q=257$}
    \end{subfigure}

    \vspace{1em} %
    
    \begin{subfigure}[b]{0.32\linewidth}
        \centering
        \HeatmapInclude[width=\linewidth, trim={0.5cm 0.5cm 0.5cm 0.5cm}, clip]{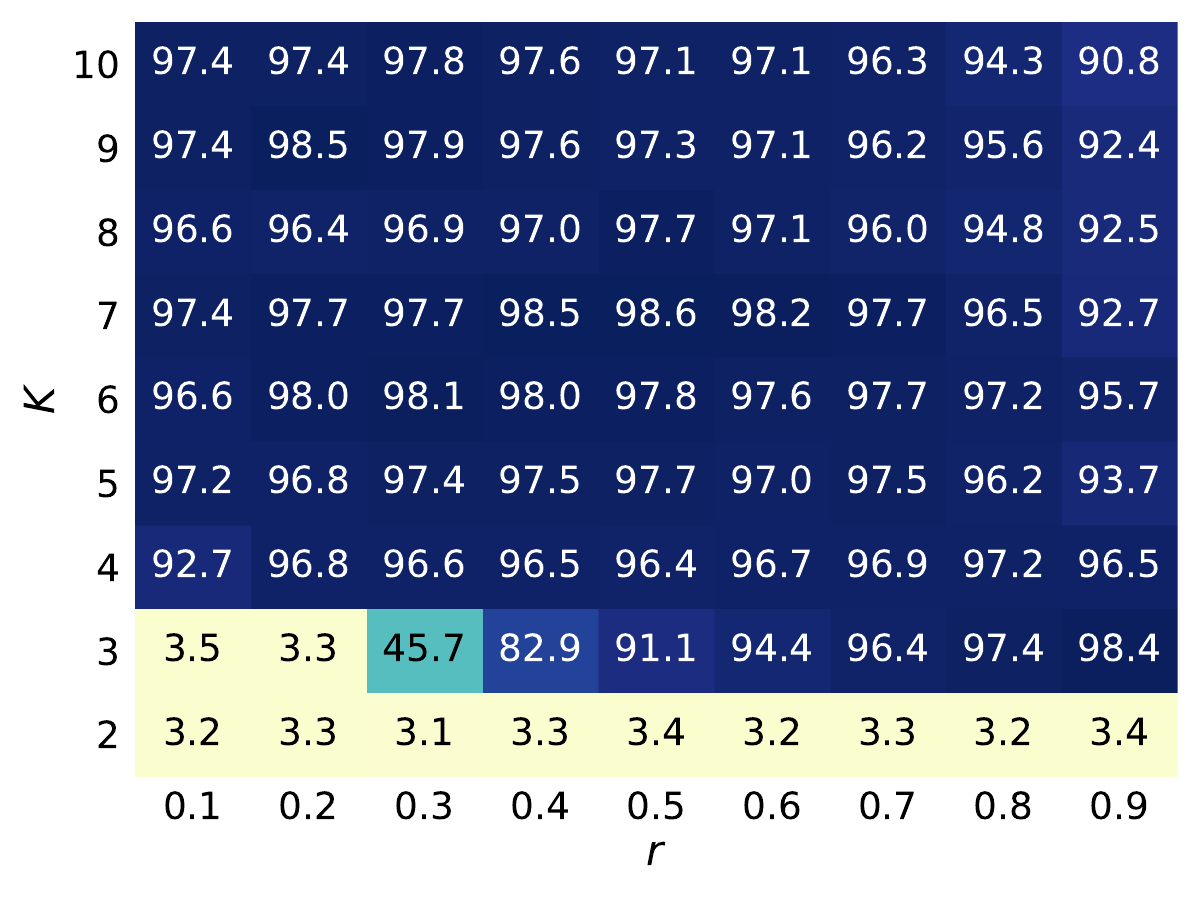}
        \caption{$N=32, q=31$}
    \end{subfigure}
    \hfill
    \begin{subfigure}[b]{0.32\linewidth}
        \centering
        \HeatmapInclude[width=\linewidth, trim={0.5cm 0.5cm 0.5cm 0.5cm}, clip]{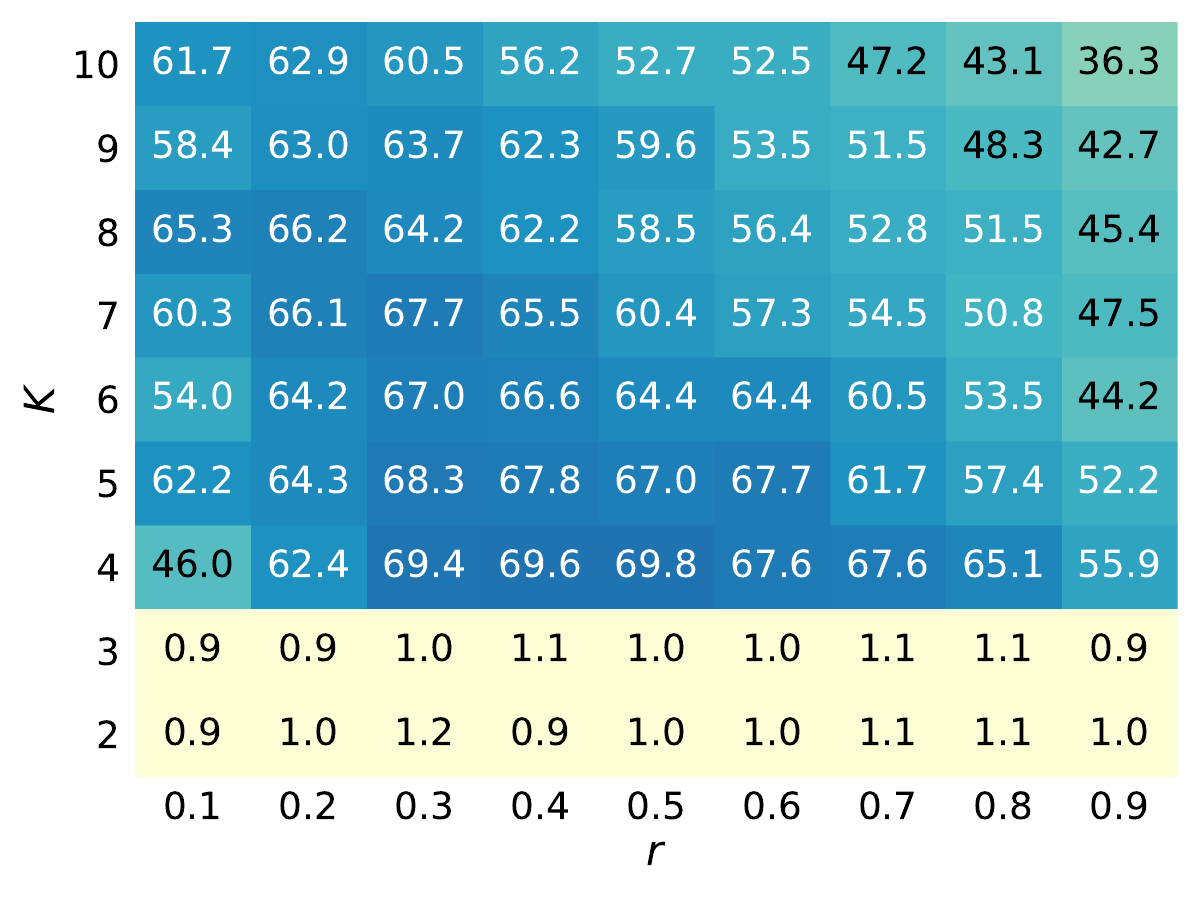}
        \caption{$N=32, q=97$}
    \end{subfigure}
    \hfill
    \begin{subfigure}[b]{0.32\linewidth}
        \centering
        \HeatmapInclude[width=\linewidth, trim={0.5cm 0.5cm 0.5cm 0.5cm}, clip]{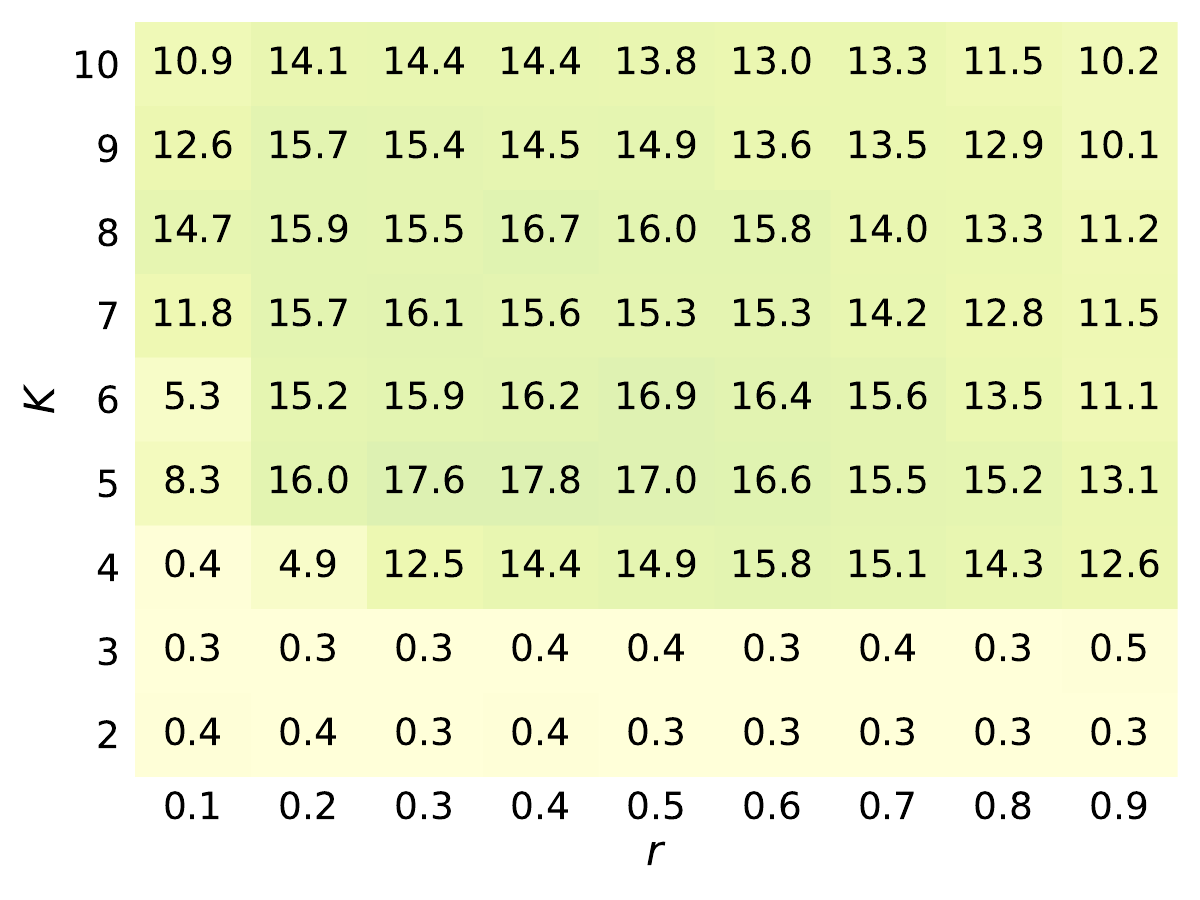}
        \caption{$N=32, q=257$}
    \end{subfigure}

    \vspace{1.5em} %
    
    \includegraphics[width=0.6\linewidth]{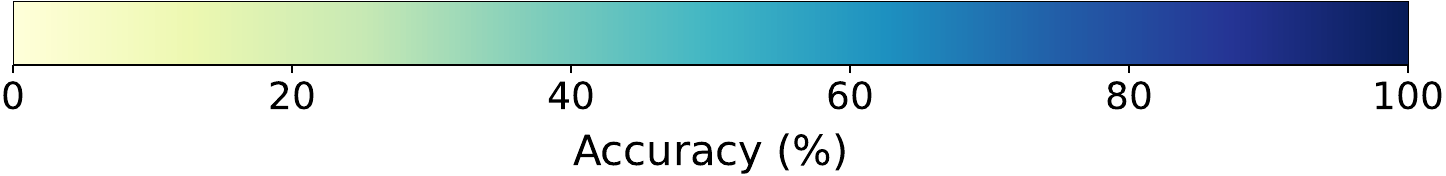} %
    
    \caption{Heatmaps of match accuracy across various combinations of $K$ and $r$ for \textbf{token embedding}. The color scale represents the match accuracy (\%), which is shared across all nine configurations.}
    \label{fig:all_heatmaps_token}
\end{figure}

\begin{figure}[htbp]
    \centering
    
    \begin{subfigure}[b]{0.32\linewidth}
        \centering
        \HeatmapInclude[width=\linewidth, trim={0.5cm 0.5cm 0.5cm 0.5cm}, clip]{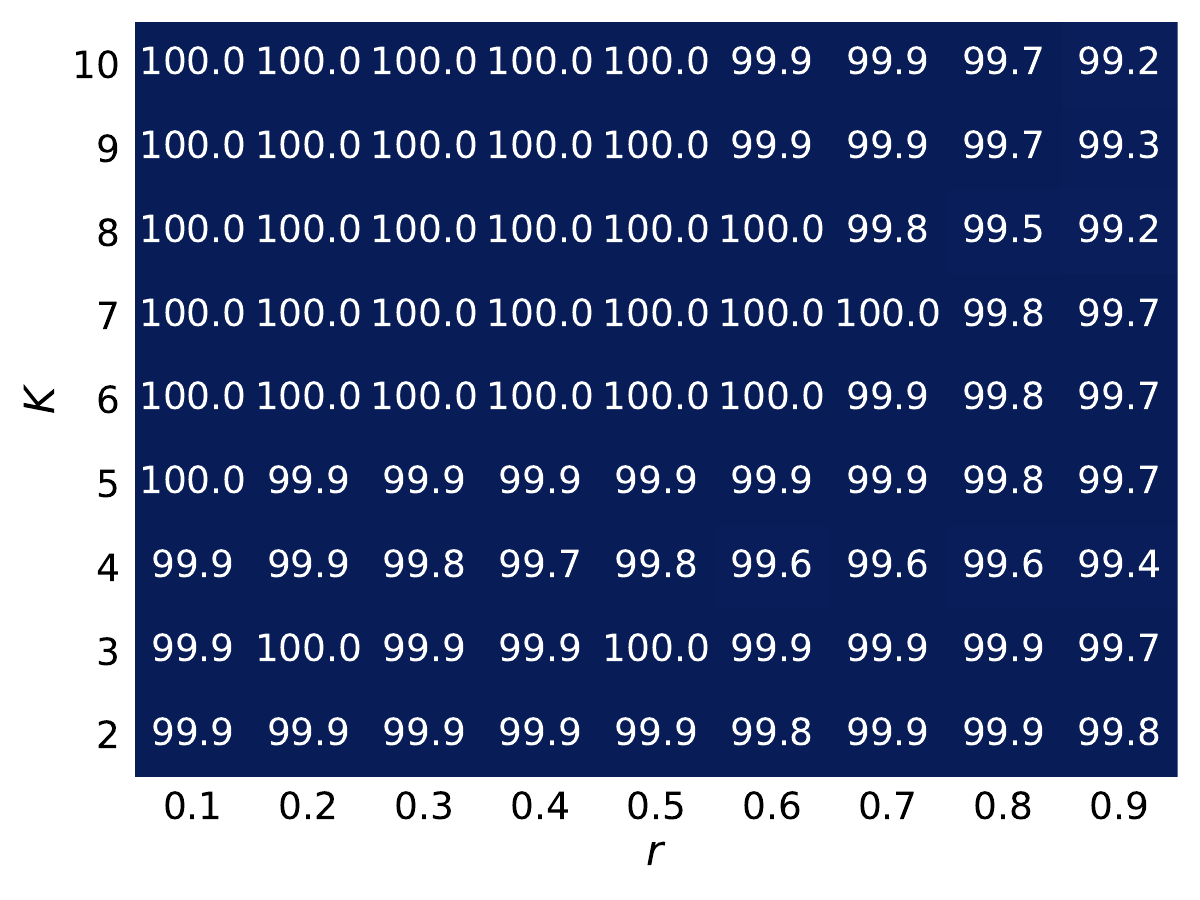}
        \caption{$N=16, q=97$}
    \end{subfigure}
    \hfill
    \begin{subfigure}[b]{0.32\linewidth}
        \centering
        \HeatmapInclude[width=\linewidth, trim={0.5cm 0.5cm 0.5cm 0.5cm}, clip]{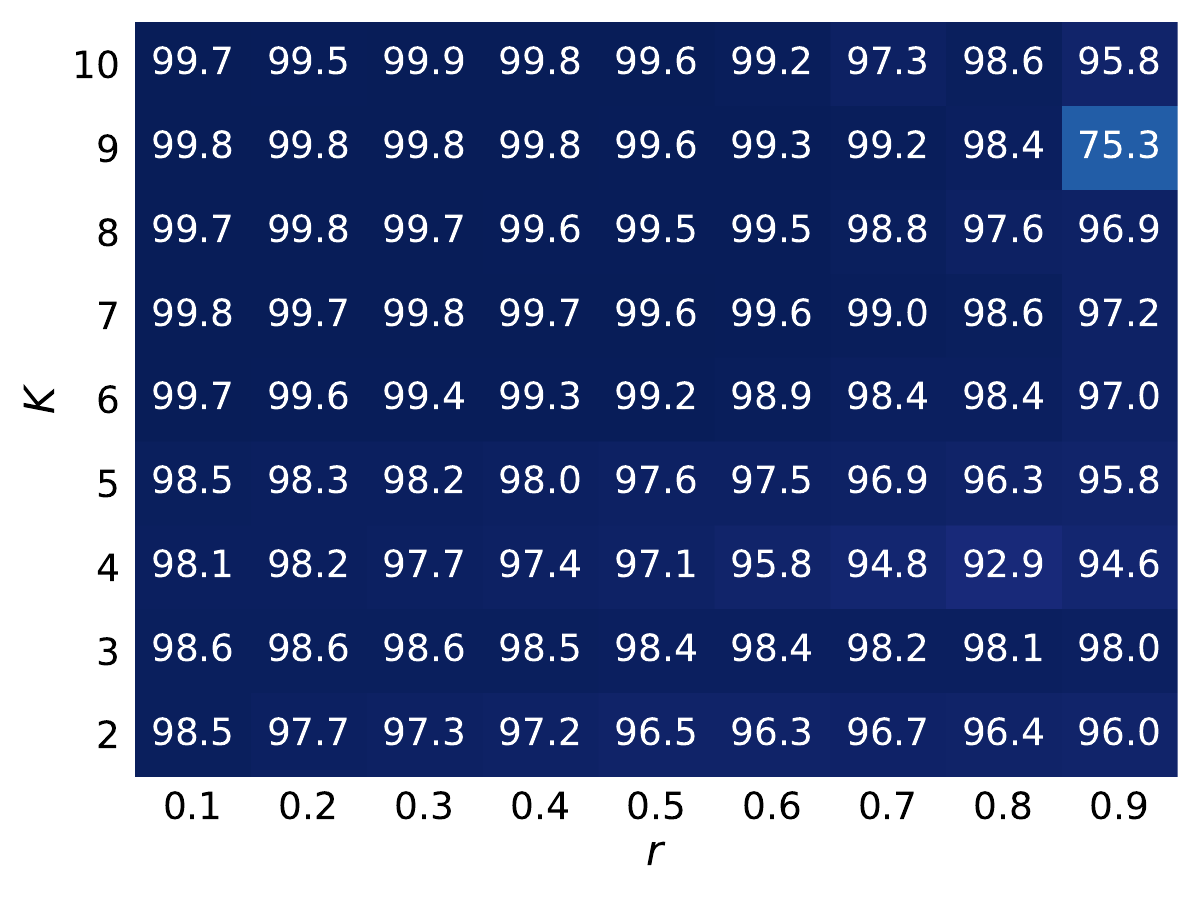}
        \caption{$N=16, q=257$}
    \end{subfigure}
    \hfill
    \begin{subfigure}[b]{0.32\linewidth}
        \centering
        \HeatmapInclude[width=\linewidth, trim={0.5cm 0.5cm 0.5cm 0.5cm}, clip]{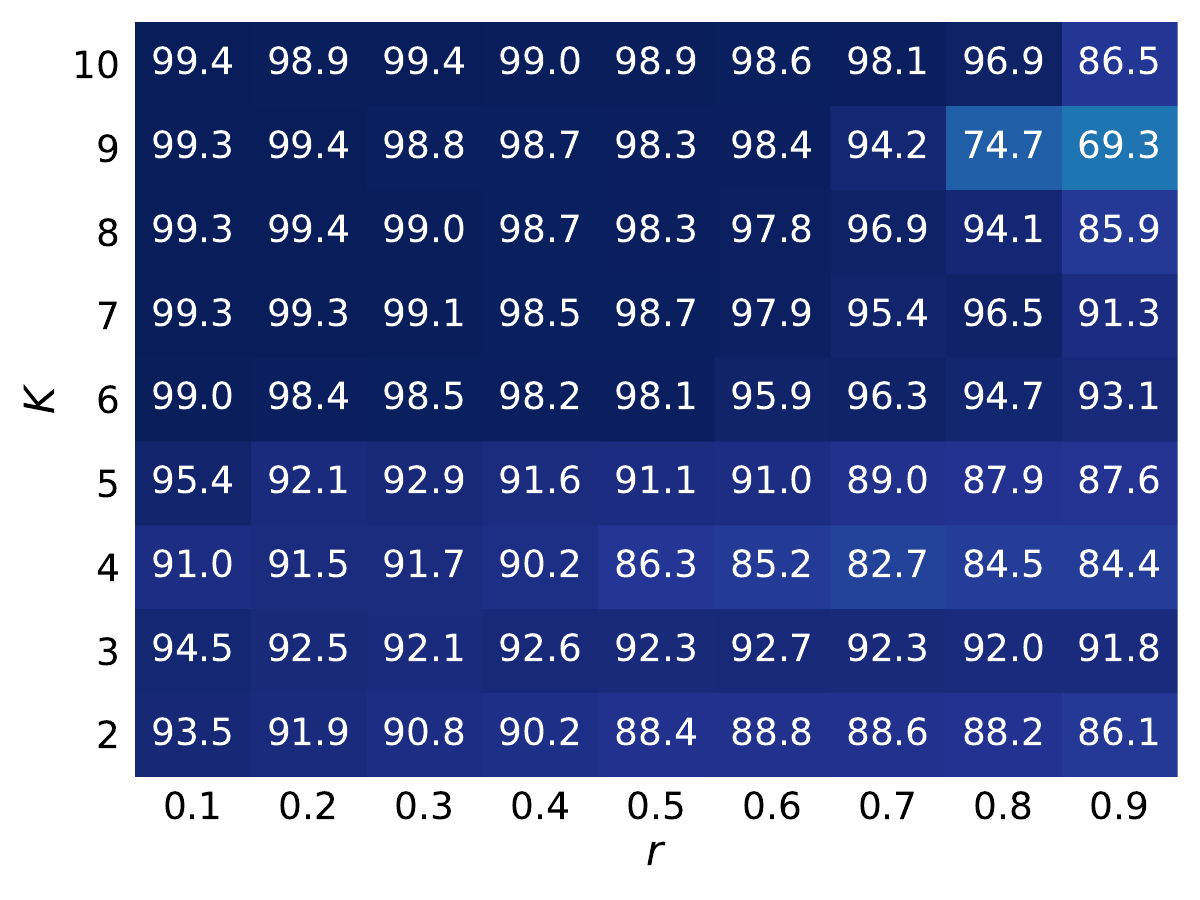}
        \caption{$N=16, q=433$}
    \end{subfigure}
    
    \vspace{1em}
    
    \begin{subfigure}[b]{0.32\linewidth}
        \centering
        \HeatmapInclude[width=\linewidth, trim={0.5cm 0.5cm 0.5cm 0.5cm}, clip]{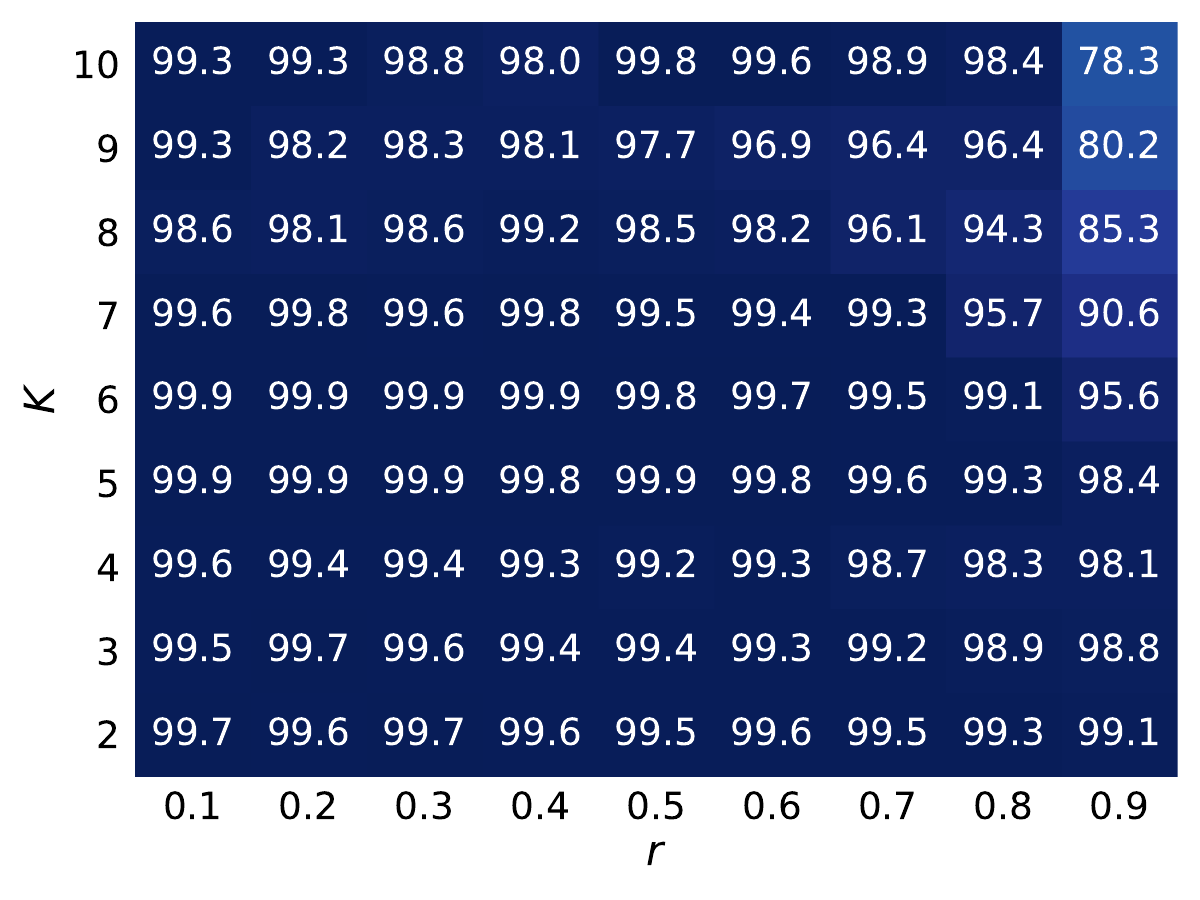}
        \caption{$N=32, q=97$}
    \end{subfigure}
    \hfill
    \begin{subfigure}[b]{0.32\linewidth}
        \centering
        \HeatmapInclude[width=\linewidth, trim={0.5cm 0.5cm 0.5cm 0.5cm}, clip]{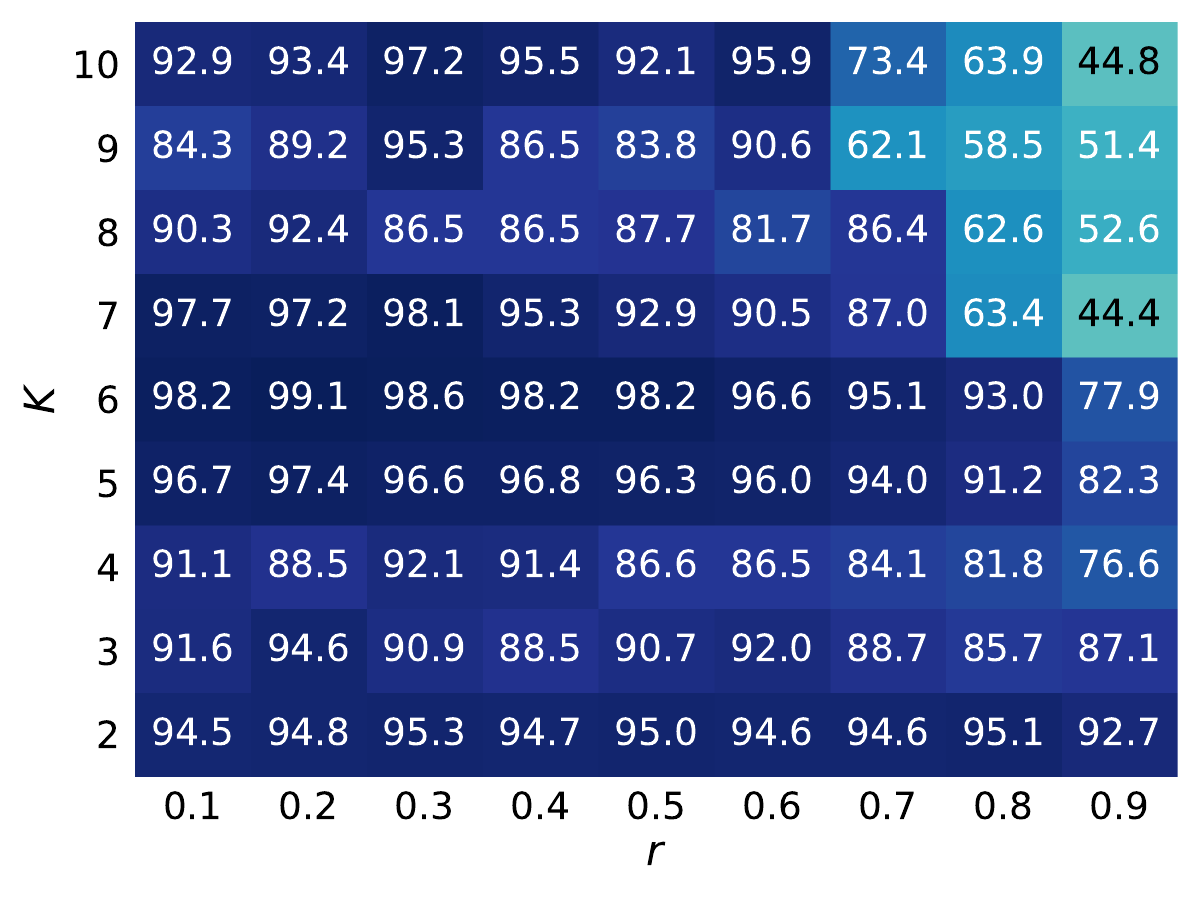}
        \caption{$N=32, q=257$}
    \end{subfigure}
    \hfill
    \begin{subfigure}[b]{0.32\linewidth}
        \centering
        \HeatmapInclude[width=\linewidth, trim={0.5cm 0.5cm 0.5cm 0.5cm}, clip]{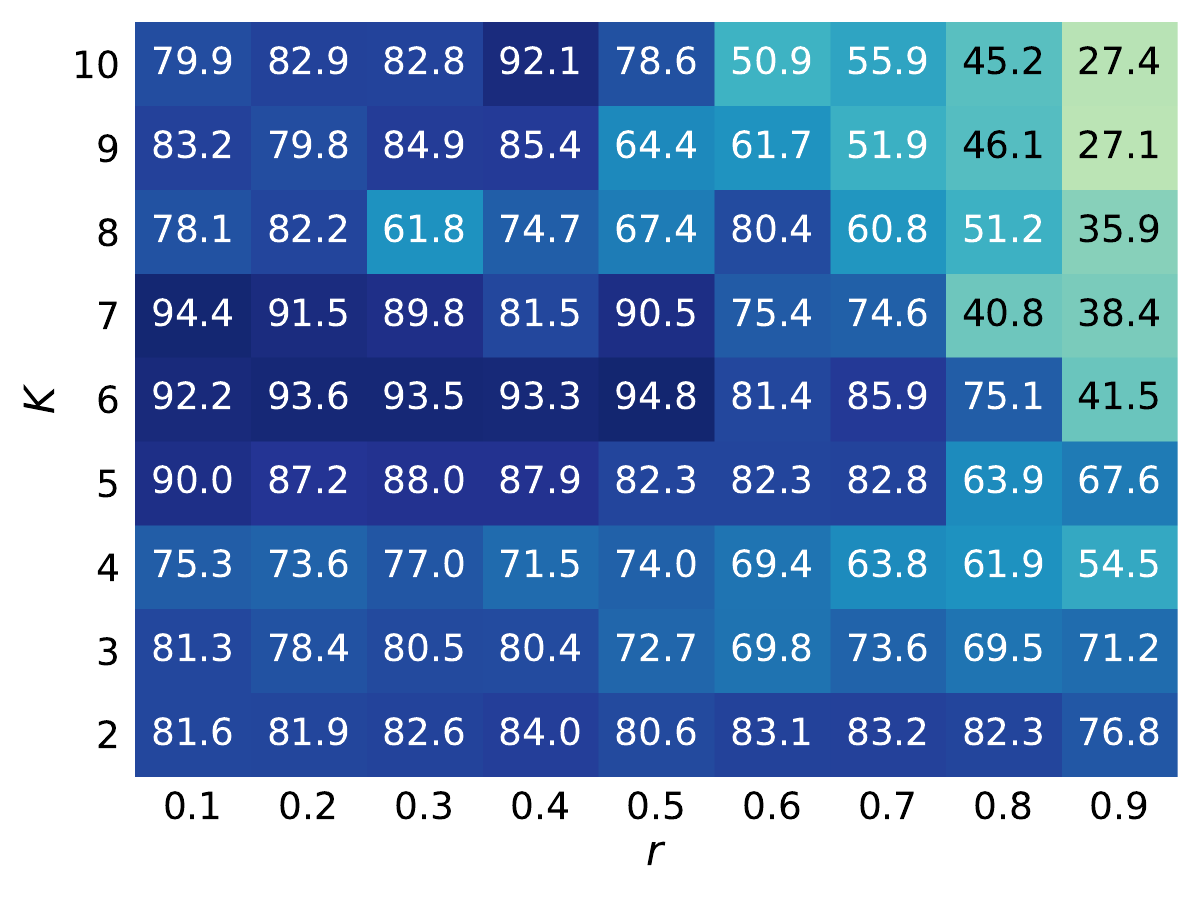}
        \caption{$N=32, q=433$}
    \end{subfigure}

    \vspace{1em}
    
    \begin{subfigure}[b]{0.32\linewidth}
        \centering
        \HeatmapInclude[width=\linewidth, trim={0.5cm 0.5cm 0.5cm 0.5cm}, clip]{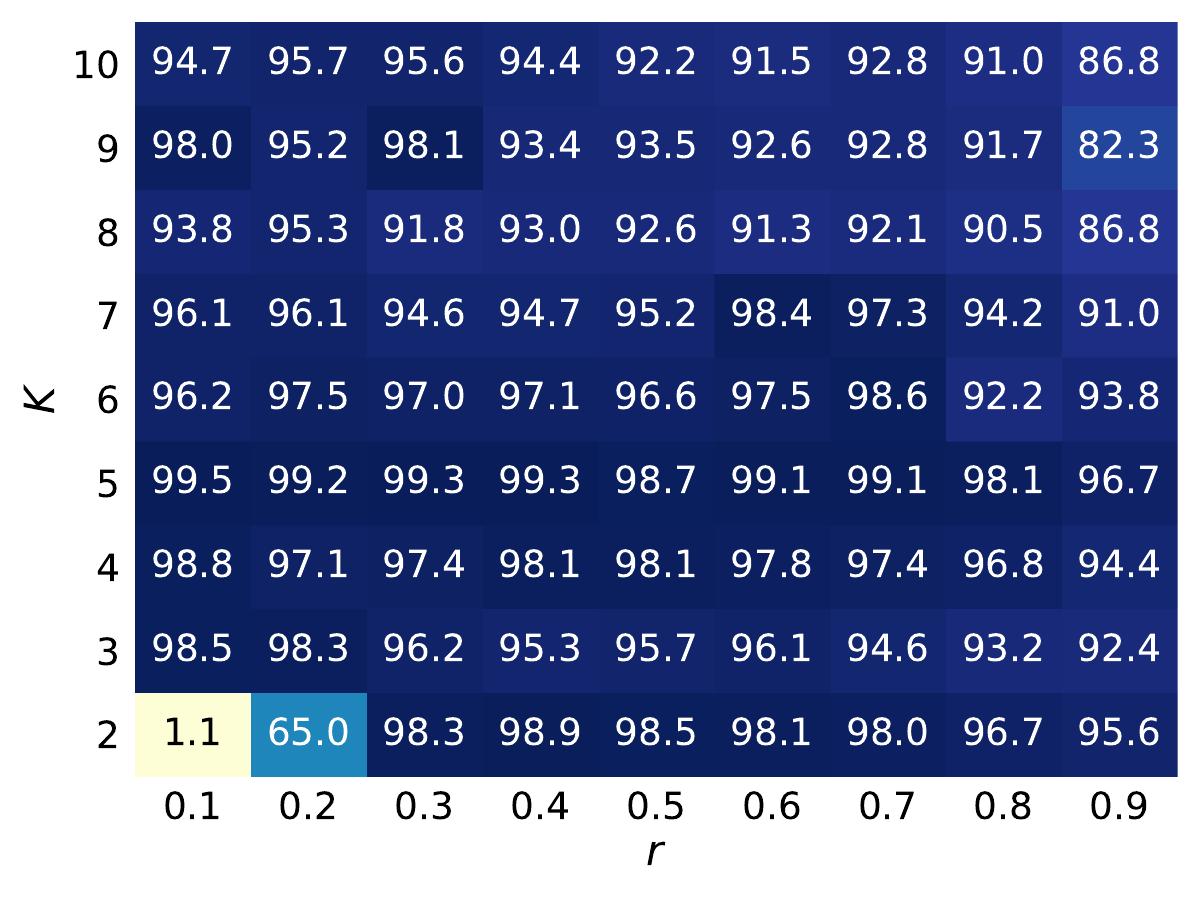}
        \caption{$N=64, q=97$}
    \end{subfigure}
    \hfill
    \begin{subfigure}[b]{0.32\linewidth}
        \centering
        \HeatmapInclude[width=\linewidth, trim={0.5cm 0.5cm 0.5cm 0.5cm}, clip]{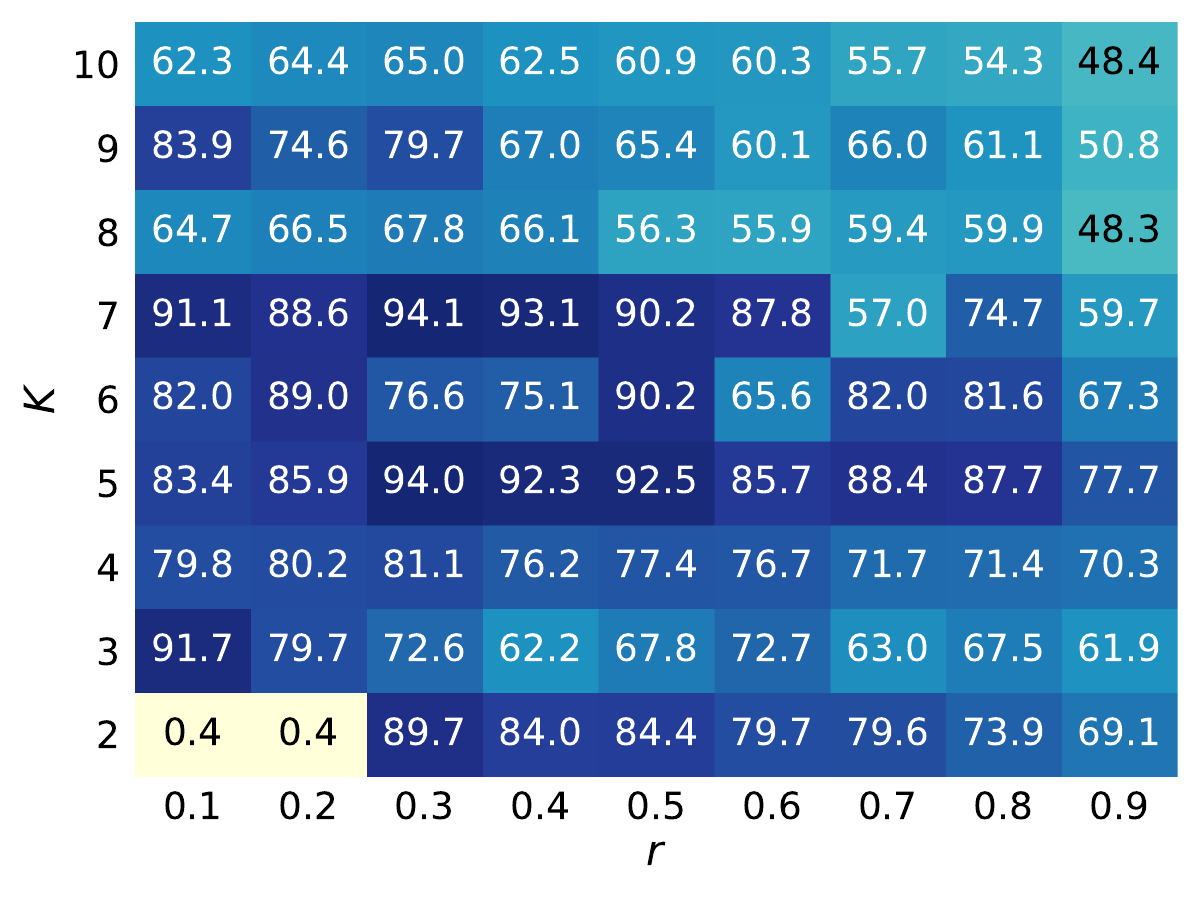}
        \caption{$N=64, q=257$}
    \end{subfigure}
    \hfill
    \begin{subfigure}[b]{0.32\linewidth}
        \centering
        \HeatmapInclude[width=\linewidth, trim={0.5cm 0.5cm 0.5cm 0.5cm}, clip]{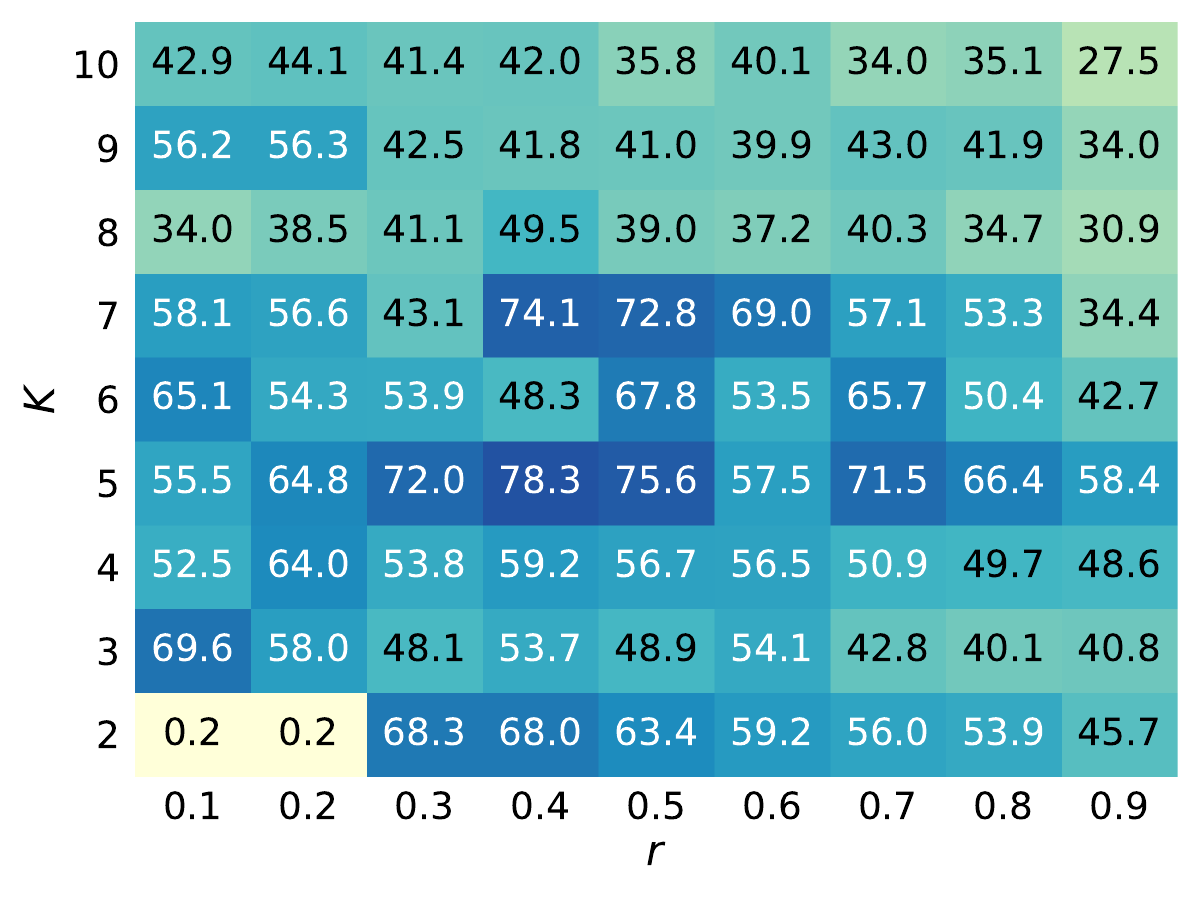}
        \caption{$N=64, q=433$}
    \end{subfigure}

    \vspace{1.5em}
    
    \includegraphics[width=0.6\linewidth]{figs/colorbar_horizontal.pdf} %
    
    \caption{Heatmaps of match accuracy across various combinations of $K$ and $r$ for \textbf{angular embedding}. The color scale represents the match accuracy (\%), which is shared across all nine configurations.}
    \label{fig:all_heatmaps_angular}
\end{figure}

In \cref{ex:ex3} of the main text, we demonstrated the robustness of our proposed method by showing its high average match accuracy across the practical grid search space ($K \in \{4, \dots, 9\}, r \in \{0.1, \dots, 0.4\}$). To provide a more comprehensive view of the hyperparameter landscape, this section presents the complete heatmaps expanded over a broader search space: $K \in \{2, \dots, 10\}$ and $r \in \{0.1, \dots, 0.9\}$. 

\Cref{fig:all_heatmaps_token,fig:all_heatmaps_angular} display the full set of grid search results for token embedding and angular embedding, respectively. As discussed in the main text, the proposed method consistently achieves high accuracy across a wide region. A significant decrease in match accuracy is primarily observed only under extreme conditions, such as when $K$ is too small (e.g., $K=2$, where modulus reduction is insufficient) or when $r$ is excessively large (where the auxiliary task dominates the primary task).

\ifPDFTeX
\end{CJK*}
\fi
\end{document}